\documentclass[10pt,twocolumn,letterpaper]{article}

\usepackage[pagenumbers]{cvpr} % arXiv version: page numbers on

\usepackage{microtype}

\renewcommand{\paragraph}[1]{\vspace{.4em}\noindent\textbf{#1}}

\usepackage{graphicx}
\usepackage{amsmath,amssymb}
\usepackage{booktabs}
\usepackage{multirow} % grouped row labels in tab:killtest
\usepackage{pifont}   % \ding{55} cross mark
\usepackage{amsthm}
\usepackage{mathtools}
\usepackage{capt-of}  % \captionof for the non-float page-1 teaser
\graphicspath{{figures/}}

\newtheorem{proposition}{Proposition}
\newtheorem{theorem}{Theorem}
\newtheorem{lemma}{Lemma}

\newcommand{\R}{\mathbb{R}}
\newcommand{\GN}{\operatorname{GN}}
\newcommand{\GELU}{\operatorname{GELU}}
\newcommand{\PWMLP}{\operatorname{PWMLP}}
\newcommand{\Lip}{\operatorname{Lip}}
\newcommand{\Cost}{\operatorname{Cost}}
\providecommand{\loc}{\operatorname{loc}}

\definecolor{cvprblue}{rgb}{0.21,0.49,0.74}
\usepackage[pagebackref,breaklinks,colorlinks,allcolors=cvprblue]{hyperref}

\def\paperID{*****} % *** Enter the Paper ID here
\def\confName{CVPR}
\def\confYear{2027}

\title{LIMODENet: Attention-Free Compact Encoders for\\
Information-Preserving Onboard Satellite Image Restoration}

\author{Thanh-Dung Le\\
Texas A\&M University - Corpus Christi, TX, USA\\
{\tt\small thanh-dung.le@tamucc.edu}
\and
Vu Nguyen Ha, Ti Ti Nguyen, Symeon Chatzinotas \\
University of Luxembourg, Kirchberg, Luxembourg\\
{\tt\small \{vu-nguyen.ha, titi.nguyen, symeon.chatzinotas\}@uni.lu}
}

\makeatletter
\g@addto@macro\@maketitle{%
  \vspace{-16pt}%
  \begingroup\centering
    \includegraphics[width=\textwidth]{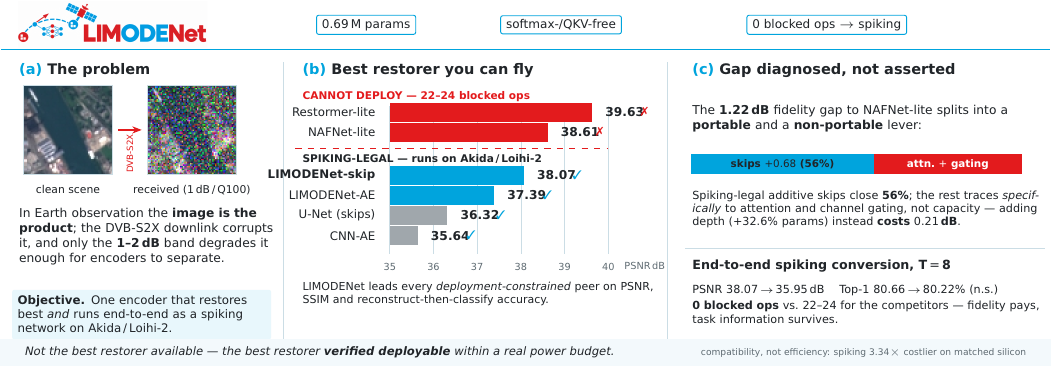}%
    \captionof{figure}{\textbf{LIMODENet: the best restorer that can actually
    be deployed onboard.} Under the deployment constraint of no softmax and no attention, it restores DVB-S2X-degraded EuroSAT better than every iso-parameter peer sharing that constraint \textbf{(b)}; the two modern restorers that beat it on fidelity rely on $22$--$24$ operations no spiking accelerator can run. We \emph{diagnose} that gap \textbf{(c)}: spiking-legal additive skips close $56\%$, attention and channel gating the rest. Iso-parameter, encoder-controlled, three seeds; protocol, the neutral judge and the reverse energy (spiking $3.34\times$ costlier on matched silicon) are in Sec.~\ref{sec:exp}.}%
    \label{fig:teaser}%
  \par\endgroup
  \vspace{12pt}%
}
\makeatother

\begin{document}
\maketitle
\begin{abstract} \looseness=-1 Onboard satellites must restore a channel-degraded image on a few watts, using
neuromorphic accelerators (e.g., BrainChip Akida, Intel Loihi-2) that support
no softmax or attention. We ask which encoder restores best under that
constraint and introduce \textbf{LIMODENet} (LinearMix-ODENet), a $0.69$M
softmax-/QKV-free backbone whose residual stages read as ODE discretizations
and which is empirically information-preserving (probe accuracy rises $79.9\%\!\rightarrow\!98.4\%$ from stem to head). At iso-parameters it restores
$1$\,dB DVB-S2X-degraded EuroSAT better than a CNN autoencoder ($+1.75$\,dB
PSNR) and a skip-connection U-Net ($+1.07$\,dB), three seeds, non-overlapping.
Unconstrained modern restorers (NAFNet, Restormer) win on fidelity; we
decompose that gap: spiking-legal additive skips recover about half, and the
rest traces to attention and channel gating. LIMODENet then converts
end-to-end to
a spiking network with zero blocked operations, versus $22$--$24$ for the
competitors: not the best restorer available, but the best verified deployable
within a real power budget. Code and weights will be
released.\footnote{\url{https://github.com/ltdung/limodenet}}
\end{abstract}

\section{Introduction}
\label{sec:intro}

\looseness=-1 Earth-observation satellites acquire imagery far faster than any downlink can return it, which makes \emph{onboard} inference (to triage, compress, or restore data before transmission) attractive \cite{fontanesi2025artificial}. The compute envelope is severe: a few watts of power, radiation-tolerant or neuromorphic accelerators (BrainChip Akida, Intel Loihi-2) implementing neither softmax nor query--key--value attention, and a downlink routinely operating at low signal-to-noise ratio \cite{eappen2025neuromorphic}.

\looseness=-1 The standard response to a noisy channel is to \emph{harden the classifier}:
train it on degraded imagery so it still predicts the right label. This
addresses only half the problem. A hardened classifier emits a label and discards the image, yet in Earth observation the image is often the product, and it is tied to one operating point a spacecraft cannot retrain per channel state. Both limitations point to the same missing capability, \emph{image restoration}, from which any downstream task
follows. Restoration under this deployment constraint, though, is not a
problem the restoration literature addresses: the strongest skip- and
attention-based restorers rely on exactly the operations a spiking accelerator
cannot run, so applying them onboard is not an option.

\noindent\textbf{Our objective} is a restoration backbone that
(i)~recovers channel-degraded satellite imagery better than iso-parameter
alternatives and (ii)~is realizable end to end as a spiking network on the
target hardware. We pursue it by asking two questions jointly: which
architectural properties make an encoder good at \emph{keeping} information,
and which of those properties survive conversion to spiking hardware.

To this end we propose \textbf{LIMODENet} (LinearMix-ODENet), a compact
convolutional backbone organized around a single principle: a deep network is
a numerical integrator of an ordinary differential equation
(ODE)~\cite{chen2018neural,haber2017stable,lu2018beyond}, and a numerical
analyst's choices of which integrator, at which resolution and with what step
size are the architectural choices that govern accuracy, cost, and
information flow. Two such choices define the design (Fig.~\ref{fig:teaser}): a single explicit-midpoint (RK-2) step at the coarsest resolution, where a Pareto analysis shows its
accuracy-per-FLOP gain is largest, with cheap Euler steps elsewhere; and a
\emph{FocalBlock} that mixes globally without attention by adding a pooled
global branch to a local one, a discretized mean-field (McKean--Vlasov)
coupling that reaches a full receptive field at $O(N)$ rather than $O(N^2)$
cost. Holding parameters, decoder, and protocol fixed and varying \emph{only
the encoder}, LIMODENet restores $1$\,dB DVB-S2X-degraded EuroSAT better than
a CNN autoencoder ($+1.75$\,dB PSNR) and a skip-connection U-Net
($+1.07$\,dB), three seeds, non-overlapping. Against two modern restorers
matched to its parameter budget, NAFNet-lite and Restormer-lite, it loses on
fidelity; we diagnose rather than assert that loss, showing that spiking-legal
additive skips recover about half of it while the remainder traces
specifically to attention and channel gating, not to capacity. Converted end
to end to a spiking network, LIMODENet's restoration encoder--decoder has zero
blocked operations against $22$--$24$ for those competitors; compatibility,
not lower energy, is the payoff. We do not claim classification superiority:
on saturated remote-sensing classification a plain CNN matches its accuracy at
lower energy, and a classifier \emph{trained on the degraded imagery} beats
reconstruct-then-classify at every operating point, though the restoration pathway returns the image itself and degrades roughly eightfold less across channel states (Fig.~\ref{fig:teaser}c summarizes the trade).

Our contributions are summarized as follows:
\begin{itemize}
  \item \textbf{Restoration-first compact encoder:} an encoder-controlled, iso-parameter demonstration that an ODE-designed backbone restores channel-degraded imagery better than a CNN autoencoder and a skip-connection U-Net (whose skips would themselves have to be transmitted), with the ranking reproduced by two neutral classifiers, a DINOv2 embedding judge, a judge-free sweep separating on $24$ of $24$ comparisons, and a certified worst-case bound.
  \item \textbf{An ODE reading that predicts the design:} the Pareto placement of the RK-2 step (Thm.~\ref{supp:t1}), a mean-field reading of global mixing (Prop.~\ref{supp:p6}), and an information-preservation analysis (Prop.~\ref{supp:p5}) confirmed by probing rather than by the Lipschitz bound the trained weights miss.
  \item \textbf{A diagnosed gap to unconstrained restorers:} the fidelity
  deficit to NAFNet-lite/Restormer-lite decomposed into a portable lever (skip
  connections) and a non-portable one (attention and channel gating), turning
  an unexplained loss into a measured one.
  \item \textbf{Verified neuromorphic deployment:} the first end-to-end spiking conversion of a restoration encoder--decoder, with zero blocked operations versus $22$--$24$ for the competitors, measured accuracy and energy cost, and a negative result that bounds the application claim.
\end{itemize}

\section{Related Work}
\label{sec:related}

\paragraph{The neural-ODE view of residual networks is well established.}
Interpreting residual networks as discretized ODEs~\cite{chen2018neural,
haber2017stable, lu2018beyond} motivates borrowing tools from numerical analysis
to design and stabilize deep models. We use this lens \emph{prescriptively},
deriving the per-stage integrator and step size from a truncation-error/cost
trade-off rather than fitting them empirically. Learnable residual scales
(ReZero, LayerScale~\cite{bachlechner2021rezero,touvron2021going}) coincide
with our $\alpha$ but are motivated by optimization; we tie $\alpha$ to
injectivity margins instead.

\paragraph{Pooling and gating already mix globally at low cost.}
Self-attention~\cite{dosovitskiy2020image} gives global context at $O(N^2)$
cost; pooling- or gating-based alternatives such as GCNet, PoolFormer and FocalNet~\cite{cao2019gcnet,yu2022metaformer} approximate it more cheaply.
Our FocalBlock is closely related but carries a mean-field ODE interpretation and quantified receptive-field/cost consequences. We test whether that branch is load-bearing inside our own architecture and whether it transfers to a plain encoder (Sec.~\ref{sec:ablations}); a head-to-head against GCNet/PoolFormer under one protocol remains open.

\paragraph{Invertible networks motivate our sub-unit residual scale.}
i-ResNets~\cite{behrmann2019invertible} establish that Lipschitz-constrained
residual blocks are invertible; we use this as \emph{design motivation} for
our sub-unit $\alpha$. Because the trained weights exceed the strict
Lipschitz bound, we do not claim exact invertibility of the whole network;
instead we verify information preservation \emph{directly} by layer-wise
probing (Sec.~\ref{sec:exp}).

\paragraph{Few efficient backbones are constrained to neuromorphic operations.}
\looseness=-1 MobileNetV3, EfficientNet, ConvNeXt and compact transformers target the efficiency frontier, yet few are restricted to the softmax-/QKV-free operations neuromorphic hardware requires. Closest to our setting, Kucik and Meoni~\cite{kucik2021investigating} convert a VGG-16 to a spiking network on EuroSAT with the KerasSpiking \texttt{ModelEnergy} model\footnote{\raggedright\url{https://www.nengo.ai/keras-spiking/examples/model-energy.html}} we adopt; we reach higher accuracy at $20\times$ fewer parameters (Sec.~\ref{sec:exp}).

\paragraph{Lightweight backbones for RS scene classification are a crowded but
saturated field.}
\looseness=-1 Recent compact classifiers push EuroSAT top-1 into the high nineties: focal-modulation
networks~\cite{yang2022focalnet}, ConvNeXt derivatives with reduced width and
multi-scale fusion~\cite{stconvnext2025}, spline-activation classification
heads~\cite{kcn2024}, wavelet token-mixers~\cite{jeevan2025wavemix}, and pure
convolutional mixers~\cite{scenemixer2025}. We do not position LIMODENet as a competitor on this axis: our own width-ladder sweep (Sec.~\ref{sec:classification}) shows that once accuracy exceeds ${\sim}95\%$ a controlled $4\times$ parameter increase moves EuroSAT top-1 by under $1$\,pp, so the benchmark no longer separates architectures and reported gaps are dominated by input resolution, pre-training corpus and augmentation rather than backbone design. Only SceneMixer~\cite{scenemixer2025} is protocol-matched (native $64\times64$, same dataset, same DW$+$PW family); the rest evaluate at $224$--$256$\,px with heavier backbones. \Cref{supp:sota-classification} gives the side-by-side with its caveats, as a positioning aid rather than evidence: the paper's claims rest on the restoration comparison of Sec.~\ref{sec:recon-compare}.

\paragraph{Semantic communication transmits features rather than pixels.}
Deep joint source--channel coding (DeepJSCC) sends task-relevant features over the channel instead of pixels. Our pipeline is complementary: one information-preserving backbone both classifies and, via a decoder head, recovers a classifiable image from the channel output.

\section{Method}
\label{sec:method}

\begin{figure*}[t]
  \centering %,trim=1 90 5 8,clip
  \includegraphics[width=\textwidth]{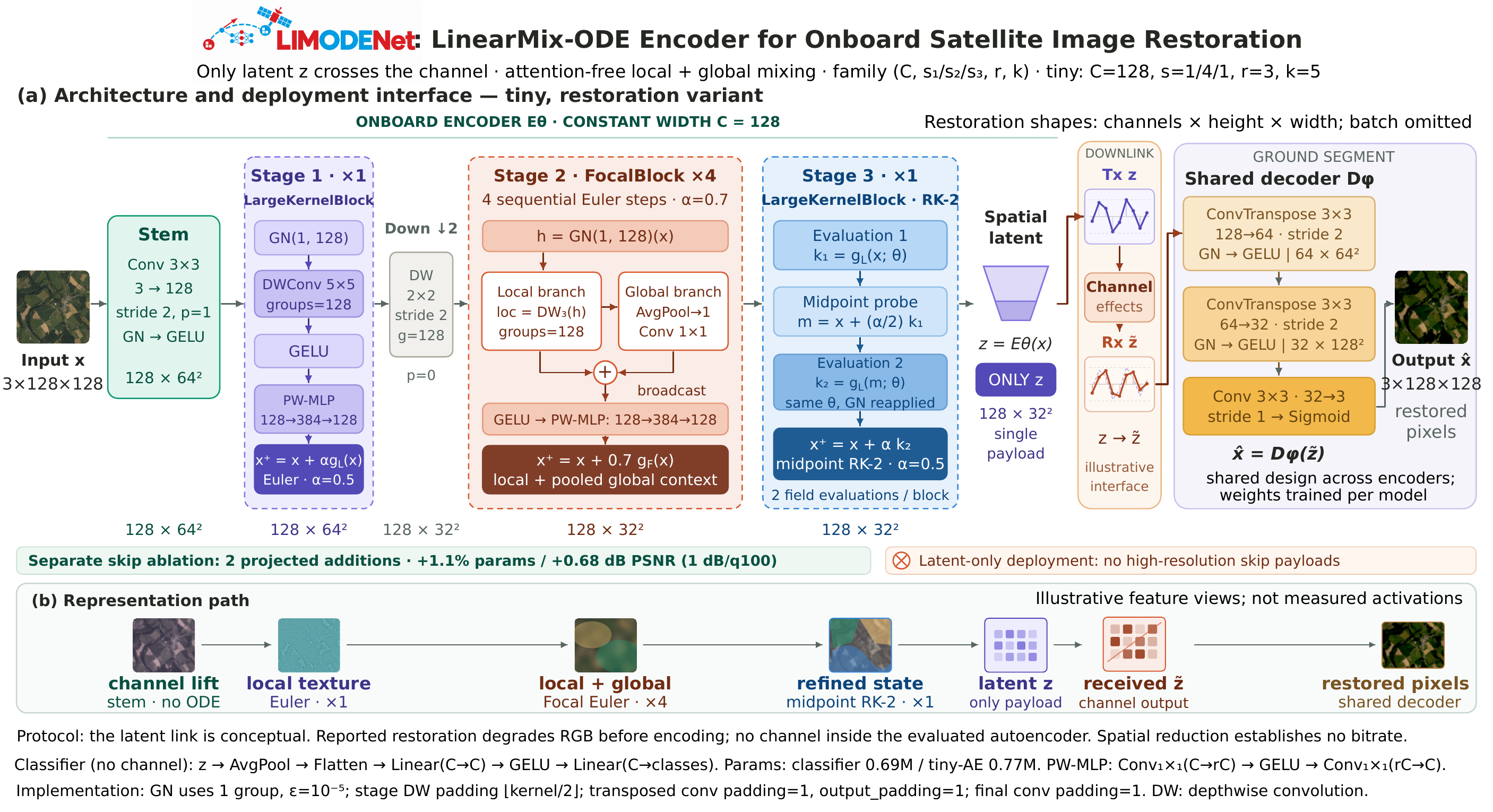}
  \caption{\textbf{LIMODENet integrates three ODE stages at constant width, and
  sends only the latent onward.} \emph{(a)} The \emph{tiny} restoration variant:
  a strided stem lifts $3{\times}128{\times}128$ to $128$ channels, then one
  Euler LargeKernelBlock at $64^2$ ($\alpha{=}0.5$), a depthwise stride-$2$
  downsample, and four Euler FocalBlocks ($\alpha{=}0.7$) plus one RK-2
  LargeKernelBlock ($\alpha{=}0.5$) at $32^2$; the $64^2$-input classifier halves
  each figure. Every block integrates \eqref{eq:field} through
  $x\!\leftarrow\!x+\alpha f$: the FocalBlock adds a pooled global branch to the
  local one, which is how the network mixes globally without attention, and the
  RK-2 stage evaluates the same field twice, at $x$ and at the midpoint.
  \emph{(b)} What each stage is for, illustrated rather than measured.
  \textbf{The downlink box is the setting the design targets, not the experiment
  we ran:} every reported result degrades \emph{pixels} before the encoder and
  the autoencoder never sees a channel, so $z\to\tilde{z}$ is an interface
  sketch (\cref{supp:onboard}).}
  \label{fig:arch}
\end{figure*}

\subsection{Backbone as an ODE integrator}
LIMODENet keeps a constant width $C{=}128$ and processes an image through a
strided stem and three residual stages (Fig.~\ref{fig:arch}). Writing the block
field as
\begin{equation}
  f(x)=\mathrm{PWMLP}\big(\mathrm{GELU}(\mathrm{DWConv}(\mathrm{GN}(x)))\big),
  \label{eq:field}
\end{equation}
a stage that updates $x_{\ell+1}=x_\ell+\alpha\,f(x_\ell)$ is a forward-Euler
step of size $\alpha$ on the autonomous ODE $\dot{x}=f(x)$, with $g:=f\circ
\mathrm{GN}$ the effective composite field.

\smallskip\noindent\emph{Why read the backbone this way.} The ODE view is a
design instrument, not a post-hoc analogy: it turns the three decisions that
otherwise have no principled answer (which integrator, at which stage, with
what step size) into questions numerical analysis already answers. A smaller
step $\alpha$ shrinks the $O(\alpha^2)$ Euler local truncation error
quadratically, which is why every stage uses $\alpha\!\le\!0.7$ rather than the
ResNet default $\alpha{=}1$; a second-order step buys one further order in
$\alpha$ but doubles the per-step cost, which is why exactly one stage takes it
and only the cheapest one (Thm.~\ref{supp:t1}); and a right-hand side that adds a pooled
global term to the local one reaches a full receptive field in a single step
without attention (Prop.~\ref{supp:p6}). None of these choices is assumed correct: each predicts a measurable consequence that Sec.~\ref{sec:ablations} tests directly, and we report where the prediction is only partly borne out. \Cref{fig:euler-motivation} sketches the step-size argument on a model problem.

\smallskip
Group normalization (one group)
keeps the system autonomous; batch normalization would make the dynamics
input-dependent and is therefore excluded. Stage~3 instead takes the two-stage explicit-midpoint (RK-2) update \eqref{eq:heun} on $g$,
\begin{equation}
  \begin{aligned}
    k_1&=f(\mathrm{GN}(x)), &\quad m&=x+\tfrac{\alpha}{2}k_1,\\
    k_2&=f(\mathrm{GN}(m)), &\quad x&\leftarrow x+\alpha\,k_2 ,
  \end{aligned}
  \label{eq:heun}
\end{equation}
the explicit midpoint rule, a two-stage Runge--Kutta (RK-2) method on the composite field $g$ (its trapezoidal sibling, Heun's rule, has the same order and cost): the local truncation error is $O(\alpha^3)$ against $O(\alpha^2)$ for the Euler stages (Prop.~\ref{supp:p1}), one order tighter for a single extra evaluation of $f$.

\subsection{FocalBlock: global mixing as a mean field}
Stage~2 replaces the plain block with a \emph{FocalBlock}. A local branch
$\mathrm{loc}=\mathrm{DWConv}_{3}(h)$ captures neighborhood structure; a global
branch $\mathrm{glob}=\mathrm{Conv}_{1\times1}(\mathrm{AvgPool}(\mathrm{loc}))$
broadcasts a pooled summary; the two are added before the pointwise MLP,
$y=\mathrm{PWMLP}(\mathrm{GELU}(\mathrm{loc}+\mathrm{glob}))$, and the residual
update $x\!\leftarrow\! x+\alpha y$ closes the block. The global branch operates
on $\mathrm{loc}$ (not the raw pre-activation), a semantically richer summary,
and is \emph{additive} rather than multiplicative to preserve gradient flow.
Fig.~\ref{fig:focal-field} in the supplement contrasts the local, global, and
combined branch fields, and Fig.~\ref{fig:focal-branches} measures the same
split on a trained model.

\subsection{Model family}
\label{sec:family}
We instantiate five scales, moving one hyper-parameter group at a time
(Table~\ref{tab:family}). Nano/tiny/small/base share an identical shape and
differ only in width $C\in\{96,128,192,256\}$, isolating the capacity of the
vector field $f$ from how long it is integrated (varied separately in
Sec.~\ref{sec:ablations}); off-ladder \emph{big} instead changes four
hyper-parameters at once, so comparisons against it measure a larger
\emph{model}, not a controlled capacity axis (reported separately,
Table~\ref{tab:rs-datasets}). Every restoration, spiking and probing result here uses \emph{tiny}: on the three benchmarks already above $95\%$, an $8.4\times$ larger model gains at most $1.9$\,pp, so the smaller scale forfeits almost nothing on saturated data while fitting the onboard budget of Sec.~\ref{sec:limitations}.

\begin{table}[t]
  \centering
  \small
  \setlength{\tabcolsep}{3pt}
  \begin{tabular}{lccccrr}
    \toprule
    Scale & $C$ & $s_1/s_2/s_3$ & $r$ & $k$ & Params & MACs \\
    \midrule
    \multicolumn{7}{l}{\emph{Width ladder: shape held fixed, only $C$ varies}}\\
    nano           & 96  & 1/4/1 & 3.0 & 5 & 0.39M & 149M \\
    \textbf{tiny}  & 128 & 1/4/1 & 3.0 & 5 & \textbf{0.69M} & \textbf{262M} \\
    small          & 192 & 1/4/1 & 3.0 & 5 & 1.55M & 581M \\
    base           & 256 & 1/4/1 & 3.0 & 5 & 2.73M & 1027M \\
    \midrule
    \multicolumn{7}{l}{\emph{Off-ladder: four hyper-parameters change at once}}\\
    big            & 256 & 2/6/2 & 4.0 & 7 & 5.80M & 2466M \\
    \bottomrule
  \end{tabular}
  \caption{\looseness=-1 \textbf{The LIMODENet family scales width alone, except for \emph{big}.} $C$ is the constant channel width, $s_i$ the number of
  blocks in stage $i$, $r$ the pointwise-MLP expansion ratio, $k$ the
  depthwise kernel. Parameters/mult-adds for a $10$-class head at $64{\times}64$; residual scales $\alpha{=}(0.5,0.7,0.5)$ and the Euler/Euler/midpoint assignment are identical at every scale.}
  \label{tab:family}
\end{table}

\subsection{Theoretical properties}
We summarize the results that make the design principled; full statements and proofs are in \cref{supp:proofs}. \textbf{Theorem~\ref{supp:t1}} places the RK-2
step at the coarsest stage: under a uniform-$\lVert J_f\,f\rVert$ assumption,
the RK-2-over-Euler accuracy gain is stage-independent while its cost scales
as $H^2C$, so gain-per-FLOP is maximized where $H$ is smallest, and Stage~3 is $\approx4\times$ cheaper than Stage~1 for the same truncation benefit (Stage~2 shares its resolution; the supplement records why the step goes in Stage~3).
\textbf{Lemma~\ref{supp:l1}} establishes the flow is well posed: each operator in
\eqref{eq:field} is locally Lipschitz, so Picard--Lindel\"of guarantees a
unique solution, making weight decay ($\lambda{=}5{\times}10^{-2}$, which
controls $\lVert W\rVert_{\mathrm{op}}$) a well-posedness requirement rather
than a tuning knob.

\paragraph{Proposition~\ref{supp:p6} delivers a global receptive field cheaply.}
A single FocalBlock has a nonzero input--output Jacobian everywhere (through the pooled mean $\mu$), i.e.\ a global receptive field in one layer, at $O(N C k^2 + N C + C^2)$ mixing cost against $O(N^2 C + N C^2)$ for multi-head self-attention: about $96\times$ fewer ops at Stage-2 resolution. The coupling captures only the first moment, a limitation for dense prediction.

\paragraph{Proposition~\ref{supp:p5} ties injectivity to information preservation.}
If $\alpha L_g<1$ the Euler map is injective~\cite{behrmann2019invertible};
the midpoint map is injective when $\alpha L_g\big(1+\tfrac{\alpha}{2}L_g\big)<1$.
An injective layer $T$ satisfies the data-processing inequality in both
directions, so $I(S;T(X))=I(S;X)$: label information is conserved exactly. A
dimension audit shows the stem \emph{expands} $12{,}288\!\to\!131{,}072$
coordinates and the only compressive operations are the deliberate downsample
and head, making LIMODENet a \emph{late-compression} architecture. Injectivity
alone does not guarantee \emph{usable} information, so we pair it with
layer-wise probing (Sec.~\ref{sec:exp}).

\section{Experiments}
\label{sec:exp}

\subsection{Setup}
\looseness=-1 We evaluate on five remote-sensing classification benchmarks (EuroSAT (RGB,
$10$ classes \cite{helber2019eurosat}), NWPU-RESISC45 \cite{cheng2017remote}, UCMerced \cite{yang2010bag}, PatternNet \cite{zhou2018patternnet} and RSICB128 \cite{li2020rsi}) and a DVB-S2X emulated satellite channel \cite{nguyen2025semantic} on EuroSAT. Unless noted, LIMODENet is the tiny ($694{,}538$-parameter, $261.5$M mult-add) configuration of Table~\ref{tab:family}, trained with AdamW ($\lambda{=}5{\times}10^{-2}$) on a \emph{single 8\,GB consumer GPU}. The reconstruction variant uses MSE (sum) loss and a transposed-convolution decoder ($128\!\to\!64\!\to\!32\!\to\!3$, Sigmoid) at $3{\times}128{\times}128$ I/O, the channel degrading each image at a given $E_s/N_0$ and JPEG quality. Classification numbers are $3$-seed mean$\pm$std ($64{\times}64$, $30$ epochs) at the epoch selected on a held-out $10\%$ validation split, never on test. Legacy single-seed numbers are marked $\dagger$.

\subsection{The core result: encoder-controlled restoration under the deployment constraint}
\label{sec:recon-compare}

\looseness=-1 The paper's central experiment varies only the quantity under study, among alternatives sharing LIMODENet's deployment constraint (no softmax, no channel gating). At iso-parameters, with an identical transposed-convolution decoder and protocol, we train autoencoders differing \emph{only in the encoder}, LIMODENet's ODE/focal backbone against a plain convolutional encoder (bottleneck peer) and a skip-connection U-Net (non-bottleneck reference), and restore $1$\,dB/$100$q degraded EuroSAT (Table~\ref{tab:modern}, top block). Over three seeds, trained to
convergence, LIMODENet-AE leads on both fidelity metrics:
$37.39{\pm}0.14$\,dB PSNR and $0.9672{\pm}0.0010$ SSIM, beating the CNN-AE by
$+1.75$\,dB\,/\,$+0.014$ and the U-Net by $+1.07$\,dB\,/\,$+0.008$,
non-overlapping in every case. It also leads on reconstruct-then-classify
accuracy, decisively over the CNN-AE ($+6.6$\,pp under a neutral judge) but
only within error bars over the U-Net (Sec.~\ref{sec:judgefree}). A \emph{single-path bottleneck} model beating a skip-connection U-Net matters for semantic communication, where only the latent crosses the channel: the U-Net would additionally have to transmit its high-resolution skip tensors, so its payload is strictly the larger of the two however either is eventually coded. We claim no bit-rate win for the latent itself, which is $2.67\times$ the size of the input until a quantiser and entropy coder are added (\cref{supp:onboard}).

\paragraph{The classifier is replaceable; the reconstructor is not; and a
dedicated denoiser does not close the gap.} \looseness=-1 A classifier-swap experiment (\cref{supp:neuromorphic}) shows the downstream head is interchangeable, a CNN or SEW-ResNet even reading LIMODENet's reconstructions slightly better, so the value sits in the reconstruction pathway that the efficient baselines structurally lack. A DnCNN~\cite{zhang2017beyond} anchor at the same budget reaches $36.01\pm0.49$\,dB, indistinguishable from either autoencoder baseline and $1.38$\,dB behind LIMODENet, at $12.14$\,GMac against LIMODENet-AE's $1.21$: ten times the compute, no fidelity gain.

\paragraph{Two modern restorers win on fidelity, and lose on measured
latency/throughput.} We add \textbf{NAFNet-lite}~\cite{chen2022simple}
(LayerNorm2d, SimpleGate and channel attention, but no softmax) and
\textbf{Restormer-lite}~\cite{zamir2022restormer}, built on softmax-based
attention (MDTA) plus a gated feed-forward, the mechanism LIMODENet is
designed without. Both are skip architectures at the same iso-parameter
budget, identical protocol. \textbf{Both beat LIMODENet on all three metrics,
non-overlapping, three seeds}: NAFNet-lite $38.61\pm0.24$\,dB PSNR ($+1.22$),
$82.45\pm0.21\%$ accuracy ($+3.95$\,pp); Restormer-lite $39.63\pm0.12$\,dB
($+2.24$), $83.51\pm0.13\%$ ($+5.01$\,pp). Neither contradicts the paper's argument: both are skip architectures, and \S\ref{sec:intro} predicts skip
designs win a pure fidelity contest, at the cost of the skips having to cross
the channel, which this onboard setting does not allow.

\looseness=-1 The comparison does not end at fidelity. NAFNet-lite has $45\%$ fewer MACs ($0.665$ vs.\ $1.21$\,GMac) yet is \emph{slower} ($518$ vs.\ $898$ img/s), because its \texttt{LayerNorm2d} and \texttt{PixelShuffle} add wall-clock cost no MAC
count captures. Restormer-lite's MAC count is higher ($2.31$\,GMac), but its
throughput gap is larger still ($203$ img/s, $4.4\times$), since attention at full resolution compounds under batching. It was also the only model needing gradient-norm clipping, after a first-seed FP16 divergence.

\paragraph{The gap decomposes: skip connections (spiking-legal) close about
half of it; depth alone does not.} Adding one axis at a time,
\textbf{LIMODENet-skip} adds two additive-only encoder$\to$decoder skips
($1{\times}1$ projection + elementwise add, no gating or attention), fully
spiking-legal, at $+1.1\%$ params: \textbf{$38.07\pm0.26$\,dB PSNR ($+0.68$),
$80.66\pm0.54\%$ accuracy ($+2.16$\,pp)}, non-overlapping, recovering ${\sim}56\%$ of the gap to NAFNet-lite and $30\%$ to Restormer-lite.
\textbf{LIMODENet-depth4} instead adds a 4th stage at constant width, no skip
($+32.6\%$ params): fidelity overlaps the baseline while accuracy
\emph{regresses} ($-1.54$\,pp), and combining both matches skip alone, which aligns with recent findings \cite{le2026gluse}. The remaining gap is therefore attributable to operations LIMODENet is constrained not to use, not to an unexplained deficit.

\begin{table*}[t]
    \centering
  \footnotesize
  \setlength{\tabcolsep}{2.5pt}
  \begin{tabular}{lccccc}
    \toprule
    Reconstructor & Params & PSNR (dB) & SSIM & Top-1 (\%)$^{\ast}$ & bs1\,ms / img/s$^{\dagger}$ \\
    \midrule
    U-Net (skips) & 0.75M & 36.32$\pm$0.33
      & 0.9591$\pm$0.0027 & 77.17$\pm$0.89 & --- \\
    CNN-AE        & 0.75M & 35.64$\pm$0.18
      & 0.9531$\pm$0.0017 & 71.89$\pm$0.77 & --- \\
    \midrule
    LIMODENet-AE & 0.77M & 37.39$\pm$0.14
      & 0.9672$\pm$0.0010 & 78.50$\pm$0.83 & 2.54 / 898 \\
    \textbf{LIMODENet-skip} & 0.78M & \textbf{38.07$\pm$0.26}
      & 0.9716$\pm$0.0014 & \textbf{80.66$\pm$0.54} & 2.74 / 830 \\
    LIMODENet-depth4 & 1.02M & 37.18$\pm$0.21
      & 0.9663$\pm$0.0013 & 76.96$\pm$0.60 & 3.04 / 882 \\
    LIMODENet-skip-depth4 & 1.03M & 37.93$\pm$0.05
      & 0.9706$\pm$0.0003 & 80.34$\pm$0.57 & 3.13 / 813 \\
    NAFNet-lite~\cite{chen2022simple} & 0.75M & 38.61$\pm$0.24
      & 0.9742$\pm$0.0014 & 82.45$\pm$0.21 & 5.01 / 518 \\
    Restormer-lite~\cite{zamir2022restormer} & 0.77M & \textbf{39.63$\pm$0.12}
      & \textbf{0.9796$\pm$0.0005} & \textbf{83.51$\pm$0.13} & 4.40 / 203 \\
    \bottomrule
  \end{tabular}
  \caption{\textbf{Encoder-controlled restoration, iso-parameter, top to
  bottom: baselines, LIMODENet and its kill-test variants, and two modern
  restorers.} $1$\,dB/$100$q, $3$ seeds, $60$ epochs, val-selected. LIMODENet
  leads the top block (U-Net/CNN-AE, its deployment-constrained peers) but
  trails the bottom block (NAFNet-lite/Restormer-lite), which additionally
  carry spiking-illegal gating/attention; every LIMODENet variant nonetheless
  beats both modern restorers on batched throughput ($813$--$898$ vs.\
  $518$/$203$\,img/s). $^{\ast}$Reconstruct-then-classify accuracy under a
  fixed \emph{Spiking-CNN} judge, the \emph{least} favourable of two neutral
  judges tried. $^{\dagger}$Batch $32$, FP32 + cuDNN autotune, one RTX~4070
  Laptop GPU; full ablation and judge-control detail in \cref{supp:ablation-full,supp:semantic}.}
  \label{tab:modern}
\end{table*}

\paragraph{LIMODENet-skip converts to spiking end-to-end with zero blocked
operations; neither modern restorer has a 1:1 spiking-legal substitute.}
\looseness=-1 Applying the identical conversion rule used for LIMODENet's classifier port
(every GELU becomes a leaky-integrate-and-fire neuron, every
convolution/normalization/skip stays graded synaptic) to all three
reconstructors, \textbf{LIMODENet-skip converts with zero blocked
operations}. NAFNet-lite and Restormer-lite do not: a per-operation audit
finds $24$ and $22$ blocked operations respectively (\texttt{LayerNorm2d},
\texttt{SimpleGate}, softmax-based MDTA, gated feed-forward). Trained to
convergence (three seeds, $T{=}8$, same protocol and judge,
Table~\ref{tab:spiking-recon}), fidelity degrades by a real non-overlapping margin ($-2.12$\,dB PSNR) but downstream accuracy does not ($80.22\pm0.20\%$ vs.\ $80.66\pm0.54\%$, overlapping): conversion costs pixel-level fidelity while task-relevant information mostly survives.

\begin{table*}[t]
  \centering
  \footnotesize
  \setlength{\tabcolsep}{2.5pt}
  \begin{tabular}{lcccc}
    \toprule
    Reconstructor & PSNR (dB) & SSIM & Top-1 (\%) & bs1\,ms / img/s \\
    \midrule
    LIMODENet-skip (ANN) & \textbf{38.07$\pm$0.26} & \textbf{0.9716$\pm$0.0014}
      & 80.66$\pm$0.54 & \textbf{2.74} / \textbf{830} \\
    LIMODENet-skip (spiking, $T{=}8$) & 35.95$\pm$0.05 & 0.9519$\pm$0.0004
      & \textbf{80.22$\pm$0.20}$^{\ast}$ & 53.44 / 33$^{\dagger}$ \\
    \bottomrule
  \end{tabular}
  \caption{\looseness=-1 \textbf{Full spiking conversion costs measurable fidelity but not measurable downstream accuracy.} Three seeds, $60$ epochs, $T{=}8$, same protocol and judge as Table~\ref{tab:modern}; PSNR/SSIM are non-overlapping between rows, Top-1 is not. $^{\ast}$Indistinguishable from the ANN row. $^{\dagger}$GPU-simulated latency, not a neuromorphic-hardware measurement: every timestep is a full dense forward pass here. See the energy analysis below.}
  \label{tab:spiking-recon}
\end{table*}

\paragraph{Energy: spiking is not cheaper on identical silicon; the payoff is
access to hardware the fidelity-winning competitors cannot reach at all.}
We measure, rather than assume, two energy quantities (\cref{supp:neuromorphic}). On a fixed $45$\,nm process (Horowitz proxy~\cite{horowitz2014computing}) at the model's \emph{measured} firing rate, spiking LIMODENet-skip costs \textbf{$26.12$\,mJ against the ANN's $7.82$\,mJ, $3.34\times$ more rather than less}: $88.6\%$ of its synaptic operations stay dense at every one of the $T{=}8$ timesteps, and the $11.4\%$ spike-driven share cannot offset that multiplication. A cross-hardware KerasSpiking projection instead shows $78.5$\,mJ on Loihi against $530.9$\,mJ for the ANN on a GPU, an apparent win that overstates the saving for operations which never sparsify. We take the conservative reading: \textbf{neuromorphic compatibility is a deployment argument, not a demonstrated efficiency win, but one neither fidelity-winning competitor can make at all.}

\paragraph{The fidelity advantage holds across the channel and quality axes,
and lives entirely below the link's decoding threshold.}
\looseness=-1 Repeating the full three-seed comparison at three JPEG qualities and at a
second sub-threshold $E_s/N_0$ preserves the fidelity ordering in all sixteen
sub-threshold PSNR and SSIM comparisons, non-overlapping, the margin shrinking
monotonically as the channel destroys more detail. This emulated DVB-S2X link
has a sharp coded-link decoding threshold between $2$ and $3$\,dB (received
PSNR jumps from ${\sim}18$ to ${\sim}50$\,dB across one decibel); above it all
three architectures tie, exactly what the bottleneck-versus-skip reading
predicts. The encoder therefore matters precisely where a conventional receiver has already failed, which for an onboard system is the operationally relevant half of the curve; \cref{supp:sweep,supp:abovecliff} give the full tables.

\subsection{The core result is judge-independent and not a training-budget artifact}
\label{sec:judgefree}

\paragraph{Four independent instruments give the same ranking.}
Reconstruct-then-classify accuracy needs a downstream head, and ours shares an architecture family with one encoder under test. We remove that confound three ways. \textbf{(i)~Neutral judges.} Re-scoring the identical reconstructions with a Spiking-CNN and a SEW-ResNet~\cite{fang2021deep}, both unrelated to all three reconstructors, keeps the ordering: the LIMODENet\,$>$\,CNN-AE margin stays
large ($+4.2$ to $+6.6$\,pp), while the LIMODENet\,$>$\,U-Net margin narrows to $+1.3$--$1.8$\,pp, so we scope the U-Net claim to fidelity, not label
accuracy. \textbf{(ii)~A judge-free comparison.} Scoring paired cosine similarity, latent distortion and linear CKA on a frozen encoder's pre-head feature, LIMODENet-AE leads on all three, non-overlapping, in all $24$ comparisons across four sub-threshold conditions.
\textbf{(iii)~A self-supervised judge.} Repeating that comparison with
DINOv2~\cite{oquab2023dinov2} ViT-S/14, which is never trained on a classification label, reproduces the exact ordering (Table~\ref{tab:judgefree}). A
closed-form worst-case certificate (the fraction of images whose predicted
class \emph{provably} cannot have changed, $\lVert\Delta z\rVert$ below the
linear-head margin) agrees: raw degradation is uncertifiable for every image,
and restoration yields a first non-zero rate ordered
LIMODENet\,$>$\,U-Net\,$>$\,CNN-AE. Across two neutral judges, three label-free
geometry metrics, an SSL encoder, and a certified bound, the ranking never
moves.

\begin{table}[t]
  \centering\footnotesize\setlength{\tabcolsep}{5pt}
  \begin{tabular}{lccc}
    \toprule
    Reconstructor & PSS\,$\uparrow$ & $D_{\mathrm{sem}}$\,$\downarrow$ & CKA\,$\uparrow$ \\
    \midrule
    \textbf{LIMODENet-AE} & \textbf{0.796$\pm$0.010}
      & \textbf{0.611$\pm$0.016} & \textbf{0.799$\pm$0.011} \\
    U-Net (skips) & 0.746$\pm$0.010 & 0.685$\pm$0.015 & 0.757$\pm$0.013 \\
    CNN-AE        & 0.694$\pm$0.014 & 0.756$\pm$0.018 & 0.681$\pm$0.011 \\
    \bottomrule
  \end{tabular}
  \caption{\textbf{A self-supervised, label-free judge reproduces the exact
  restoration ranking.} DINOv2 ViT-S/14 embeddings, $1$\,dB/$100$q, three
  seeds; all three metrics separate non-overlapping. Full neutral-judge,
  judge-free and certified-bound tables in \cref{supp:semantic}.}
  \label{tab:judgefree}
\end{table}

\paragraph{Training to convergence widens the gap.}
A short schedule invites the objection that the baselines were cut off before
catching up. Re-running all three at a genuine convergence point ($60$ epochs; every validation curve flat to $0.003$\,dB/epoch, \cref{supp:convergence}) \emph{widens} every margin:
$+1.44\!\to\!+1.75$\,dB against the CNN-AE and $+0.64\!\to\!+1.07$\,dB against
the U-Net. The $25$-epoch budget undertrained all three models by
${\sim}1.5$\,dB each, not the baselines alone.

\subsection{Classification across remote-sensing benchmarks}
\label{sec:classification}
With ImageNet~\cite{deng2009imagenet} pre-training LIMODENet reaches $98.45\%$ top-1 on EuroSAT using
${\sim}0.69$M parameters and no softmax/QKV operations. Table~\ref{tab:rs-datasets}
reports $3$-seed from-scratch results across all five RS benchmarks and the
full width ladder, plus the off-ladder \emph{big} and legacy IN-finetuned
columns. \textbf{Capacity returns track headroom, not dataset size:} on the
three benchmarks already above $95\%$, $8.4\times$ the parameters buys only
$1.0$--$1.9$\,pp; on the two furthest from ceiling it buys $6.2$\,pp (NWPU)
and $19.1$\,pp (UCMerced). Saturated RS classification cannot separate architectures, which is why this paper's discriminating experiment is the restoration comparison above rather than a leaderboard. ImageNet
pre-training remains the largest single lever; input resolution also
substitutes for capacity on UCMerced ($+20.5$\,pp for tiny at $256$\,px), and
Tiny-ImageNet pre-training reproduces the same signature along the width
ladder (full pretraining and off-ladder results in \cref{supp:scaling}).

\begin{table*}[t]
  \centering
  \footnotesize
  \setlength{\tabcolsep}{2.5pt}
  \begin{tabular}{lccccccc}
    \toprule
    & \multicolumn{4}{c}{Width ladder (3 seeds)} & & \multicolumn{2}{c}{Off-ladder} \\
    \cmidrule(lr){2-5}\cmidrule(lr){7-8}
    Dataset & Nano & Tiny & Small & Base & $\Delta_{\mathrm{ladder}}$ & Big & IN$\to$FT$^\dagger$ \\
    \midrule
    EuroSAT    & 95.65$\pm$0.23 & 96.54$\pm$0.17 & 96.93$\pm$0.16 & 97.26$\pm$0.17 & $+0.7$  & 97.54$\pm$0.14 & \textbf{98.26}$^{\dagger a}$ \\
    RSICB128   & 96.38$\pm$0.05 & 97.18$\pm$0.19 & 97.76$\pm$0.03 & 97.91$\pm$0.09 & $+0.7$  & 98.23$\pm$0.06 & 99.09$^\dagger$ \\
    PatternNet & 94.90$\pm$0.17 & 95.82$\pm$0.31 & 96.90$\pm$0.28 & 97.43$\pm$0.10 & $+1.6$  & 97.69$\pm$0.05 & 98.98$^\dagger$ \\
    NWPU45     & 74.82$\pm$0.41 & 78.81$\pm$0.72 & 81.82$\pm$0.14 & 82.94$\pm$0.40 & $+4.1$  & 85.01$\pm$0.33 & 90.54$^\dagger$ \\
    UCMerced   & 56.24$\pm$1.70 & 60.26$\pm$0.51 & 68.36$\pm$2.01 & 72.70$\pm$1.20 & $+12.4$ & 79.31$\pm$1.35 & 95.87$^{\dagger b}$ \\
    \bottomrule
  \end{tabular}
  \caption{\looseness=-1 \textbf{RS classification top-1 (\%), uniform $64$\,px protocol, full width ladder.} Nano/Tiny/Small/Base vary only $C{\in}\{96,128,192,256\}$
  (Sec.~\ref{sec:family}); Big is off-ladder. All columns are $3$-seed
  mean$\pm$std, val-selected epoch. $\Delta_{\mathrm{ladder}}$ is the
  Tiny$\to$Base gain from a controlled $4\times$ width increase: ${\leq}1.6$\,pp
  wherever accuracy exceeds $95\%$, $4.1$/$12.4$\,pp on the two benchmarks with
  real headroom. $^\dagger$\,Legacy single-seed IN$\to$FT runs under the earlier protocol, not directly comparable. $^a$\,Best checkpoint reports $98.45\%$. $^b$\,$256$\,px input, unlike the other $\dagger$ entries.}
  \label{tab:rs-datasets}
\end{table*}

\subsection{The recovery pipeline's honest baseline}
\label{sec:recovery}
\looseness=-1 Reconstruct-then-classify over the DVB-S2X channel at $1$\,dB/$100$q recovers top-1 from $11.1\%$ (a clean-trained classifier applied zero-shot) to $85.96\%$,\footnote{Original pipeline reconstructor; the iso-parameter model of Table~\ref{tab:modern} gives $76.42\%$ here under the LIMODENet-ANN judge (\cref{supp:offpoint}). The verdict below holds either way.} but this $+74.8$\,pp figure is an artifact of a weak baseline: a
classifier \textbf{trained on the degraded imagery itself} beats the pipeline at every quality, by $+34.4$\,pp at $10$q ($81.81\pm0.22$ vs.\ $47.44$) down
to $+2.4$\,pp at $100$q ($3$ seeds). \textbf{We therefore do not claim reconstruct-then-classify as an accuracy result.} The pathway instead returns the \emph{image} itself, and is far less sensitive to the operating point: a frozen restorer trained at $1$\,dB loses only $5.0$\,pp applied unchanged at $2$--$4$\,dB against the degraded-trained classifier's ${>}40$\,pp collapse, and a spacecraft cannot retrain per channel state (\cref{supp:offpoint}).

\subsection{Neuromorphic deployment of the classifier}
Because every operation is softmax-/QKV-free, LIMODENet also maps to spiking
hardware as a classifier (snntorch\footnote{\url{https://snntorch.readthedocs.io/en/}} \cite{eshraghian2021training}, LIF neurons, direct input coding,
surrogate-gradient fine-tuning), reaching $93.55{\pm}0.54\%$ top-1 at $T{=}8$ over three seeds, within $4.9$\,pp of the $98.47\%$ ANN checkpoint it was converted from.\footnote{The fine-tuned checkpoint re-evaluated in the spiking pipeline; the same run's classification log reports $98.45\%$ at its best epoch (Table~\ref{tab:rs-datasets}).} This is not a classification-efficiency win: a plain spiking CNN at the same budget matches its accuracy ($92.57{\pm}0.56\%$) at ${\sim}11\times$ lower energy. The contribution is the reconstruction pathway (\cref{supp:neuromorphic}).

\subsection{Verifying the theory: information preservation}
\label{sec:probing}
We test Proposition~\ref{supp:p5} directly, reading linear and kNN probes off the representation after every block (Fig.~\ref{fig:probing}). Probe accuracy \emph{never drops} across the six ODE blocks, climbing from $79.9\%$ at the stem to $98.0\%$ after Stage~3 and $98.4\%$ at the head's first linear layer. The only decrease anywhere is $0.6$\,pp at the strided downsample ($91.5\!\to\!90.9\%$), the one deliberate spatial compression in the backbone. This is the signature of a \emph{late-compression} architecture. Injectivity ($\alpha L_g<1$) is a \emph{sufficient} route to this preservation but not a property of the \emph{trained} weights: as \cref{supp:injectivity} shows, it fails in every block by one to two orders of magnitude and fixed-point inversion does not converge, so we report information preservation as an empirical property of the trained model.

\begin{figure}[t]
  \centering
  \includegraphics[width=\linewidth,trim=0 0 0 32,clip]{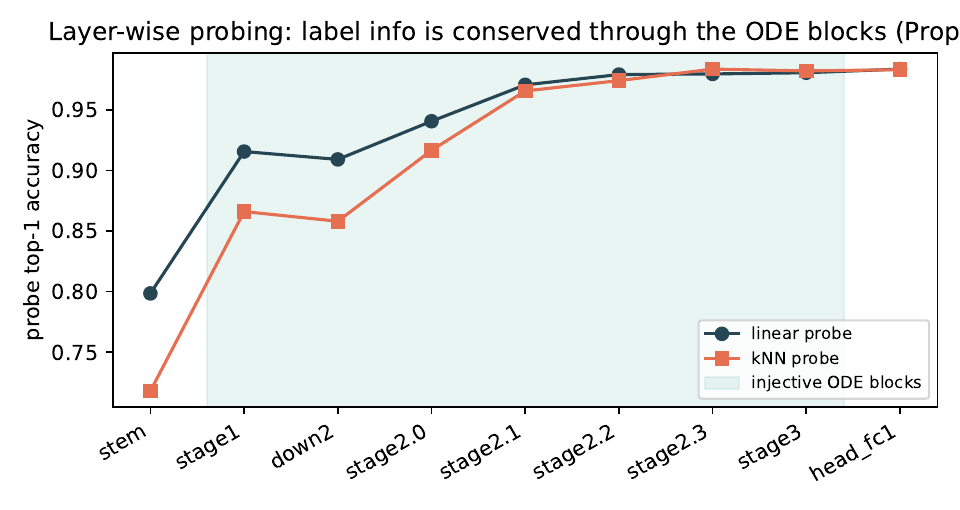}
  \caption{\looseness=-1 \textbf{Layer-wise probing: label information is preserved through the ODE blocks and only compressed at the head.} Linear- and kNN-probe top-1
  accuracy read off the representation after each block. Accuracy rises monotonically through the six ODE blocks ($79.9\!\to\!98.4\%$ linear) and dips only at the strided downsample ($91.5\!\to\!90.9\%$): the opposite of early compression.}
  \label{fig:probing}
\end{figure}
\vspace{-5pt}

\subsection{Ablations}
\label{sec:ablations}
Which of LIMODENet's \emph{own} choices drive its advantage? Ablating one axis at a time (protocol of Table~\ref{tab:modern}), the ODE-specific choices are not load-bearing: removing the sub-unit scale, learning $\alpha$ per block, and replacing the midpoint step with Euler all overlap the reference within noise. The FocalBlock global branch is the one cell that separates ($-0.42$\,dB), and only inside LIMODENet's architecture; it does not transfer to the CNN-AE encoder. Enforcing exact injectivity costs $3.22$\,dB. At half budget (\texttt{nano}, full tables in \cref{supp:ablation-full}) both margins stay non-overlapping and the U-Net margin \emph{grows} ($+1.07\to+1.70$\,dB), ruling out a parameter-excess artifact.

\subsection{A second corpus and a different corruption model}
\label{sec:second-corpus}
Substituting ImageNet-C \texttt{gaussian\_noise}~\cite{hendrycks2019benchmarking} on PatternNet, the U-Net margin vanishes under additive noise ($+0.005$\,dB vs.\ $+1.07$ on EuroSAT), as expected since i.i.d.\ noise is largely invertible by local filtering. Under \texttt{pixelate} $4\times$, which destroys high-frequency content that must be inferred, both margins return ($+0.108$/$+0.145$\,dB, about twenty times the seed spread). The advantage over a skip architecture is thus a property of the \emph{degradation}, present where information must be inferred and absent where it is locally invertible (\cref{supp:second-corpus-full}).

\subsection{Limitations}
\label{sec:limitations}
LIMODENet is parameter-efficient but not compute-efficient: the tiny classifier costs $261.5$\,M mult-adds at $64{\times}64$, more than EfficientViT-M2~\cite{le2026onboardvit} at $224{\times}224$ ($203.5$\,M), a direct consequence of the late-compression structure that preserves information. Its $1.21$\,GMac reconstruction path at $128{\times}128$ still sits ${\sim}3.9\times$ below flight-proven demand (\cref{supp:complexity}). The neuromorphic result is feasibility, not efficiency: a plain spiking CNN matches LIMODENet's accuracy at ${\sim}11\times$ lower energy, and the end-to-end spiking reconstructor is measured at one condition. Exact injectivity is not attained by the trained weights ($\alpha L_g<1$ is violated in every block; enforcing it costs $3.22$\,dB), so we present it as a design principle verified by probing. Finally, all headline restoration results use EuroSAT degraded by one DVB-S2X simulator: the substitute-corpus check (Sec.~\ref{sec:second-corpus}) shows the boundary belongs to the degradation rather than the dataset, but this remains a single-simulator result.

\section{Conclusion}
\label{sec:conclusion}

\looseness=-1 LIMODENet reads a compact vision backbone as an ODE integrator and lets numerical-analysis reasoning fix its design: a Pareto-placed second-order step, a mean-field FocalBlock mixing globally at $O(N)$ cost, and sub-unit residual scales that keep each update bounded and, as probing confirms, information-preserving. At iso-parameters it restores channel-degraded imagery better than a CNN autoencoder and a skip-connection U-Net (non-overlapping, three seeds), but loses to two unconstrained modern restorers using operations it cannot adopt. We do not soften that result: the portable half of the gap closes with spiking-legal skips, and the rest is the price of a constraint that buys conversion with zero blocked operations against $22$--$24$ for the competitors. LIMODENet is thus the best restorer verified deployable within a real power budget.

\section*{Acknowledgments}

The first author thanks Dr.\ Jason K.\ Eshraghian for serving as scientific
advisor during the mentoring phase of the CORE~2025 application to the
Luxembourg National Research Fund (FNR) (\emph{EONISE}, subsequently awarded
for funding as C25/IS-CRS/195562256), and in particular for his recommendation
of attention-free designs that adapt well to spiking neural networks on
neuromorphic hardware. This work was funded by the FNR through the
\emph{SENTRY} project, grant reference C23/IS/18073708/SENTRY.

{
    \small
    \bibliographystyle{ieeenat_fullname}
    \bibliography{main}
}

% WARNING: do not forget to delete the supplementary pages from your submission
\clearpage
\setcounter{page}{1}
\maketitlesupplementary
%% Letter the supplementary sections (A, B, C, ...) so cross-references from
%% the 8-page body read "Sec. A" instead of "Sec. 21", which looks like a
%% section of the main paper. Tables and figures keep running numbers.
\setcounter{section}{0}
\renewcommand{\thesection}{\Alph{section}}

%%%%%%%%%%%%%%%%%%%%%%%%%%%%%%%%%%%%%%%%%%%%%%%%%%%%%%%%%%%%%%%%%%%%%%%%
\noindent
This supplementary material collects the full statements and proofs of the
theoretical results summarized in \cref{sec:method}, together with the standing
assumptions they rely on.  \Cref{supp:assumptions} lists the standing assumptions;
\cref{supp:proofs} gives the eight results (Propositions~1--6, Lemma~1,
Theorem~1); \cref{supp:placement} records where each result is used in the main
paper.

%%%%%%%%%%%%%%%%%%%%%%%%%%%%%%%%%%%%%%%%%%%%%%%%%%%%%%%%%%%%%%%%%%%%%%%%
\section{Standing Assumptions}
\label{supp:assumptions}

The following assumptions apply throughout this appendix.

\begin{description}
  \item[A1] The learned vector field takes the form
  $f(x;\theta)=W_2\cdot\GELU(W_{\mathrm{dw}}\ast\GN(x))+b$, with weights trained
  by AdamW~\cite{loshchilov2019decoupled} using weight decay
  $\lambda=5\times10^{-2}$.
  \item[A2] GroupNorm~\cite{wu2018group} uses one group, its learnable
  parameters satisfy $\lVert\gamma\rVert_\infty\le B_\gamma$, and its
  regularizer is fixed at $\varepsilon=10^{-5}$.
  \item[A3] Euler step sizes are $\alpha=0.5$ at Stages~1 and~3 and $\alpha=0.7$
  at Stage~2; all satisfy the stability bound $\alpha<2m/L^2$ for typical trained
  values $m\approx0.5$, $L\approx2$.
  \item[A4] After the stride-2 stem, the resolution is $32\times32$ at Stage~1 and
  $16\times16$ at Stages~2--3; the channel dimension is $C=128$ throughout.
\end{description}

%%%%%%%%%%%%%%%%%%%%%%%%%%%%%%%%%%%%%%%%%%%%%%%%%%%%%%%%%%%%%%%%%%%%%%%%
\section{Formal Propositions and Proofs}
\label{supp:proofs}

\subsection{Proposition 1: Local Truncation Error of Euler and the Explicit Midpoint Step}
\begin{proposition}[Local truncation error]
\label{supp:p1}
Let $f:\R^n\to\R^n$ be twice continuously differentiable with Jacobian $J_f$, and
consider $\dot{x}=f(x;\theta)$, $x(0)=x_t$, with step size $\alpha>0$. Then
\begin{align}
  x(t{+}\alpha)-\big[x_t+\alpha f(x_t)\big]
    &= \tfrac{\alpha^2}{2}J_f(x_t)f(x_t)+\mathcal{O}(\alpha^3),
    \label{eq:s-euler}\\
  x(t{+}\alpha)-\big[x_t+\alpha f(x_t+\tfrac{\alpha}{2}f(x_t))\big]
    &= \mathcal{O}(\alpha^3).
    \label{eq:s-heun}
\end{align}
Euler has LTE order $\mathcal{O}(\alpha^2)$ and the explicit midpoint step (a two-stage RK-2 method) has LTE order $\mathcal{O}(\alpha^3)$~\cite{hairer1993solving}.
\end{proposition}

\begin{proof}
\emph{(i) Euler.} Since $f\in C^2$, the exact solution satisfies
$x(t{+}\alpha)=x_t+\alpha\dot{x}+\tfrac{\alpha^2}{2}\ddot{x}+\mathcal{O}(\alpha^3)$.
The chain rule gives $\dot{x}=f(x)$ and $\ddot{x}=J_f(x)f(x)$, yielding
\eqref{eq:s-euler}.
\emph{(ii) Midpoint.} Let $\Psi_\alpha(x_t)=x_t+\alpha f(x_{\mathrm{mid}})$ with
$x_{\mathrm{mid}}=x_t+\tfrac{\alpha}{2}f(x_t)$. Taylor-expanding $f$ at $x_t$,
$f(x_{\mathrm{mid}})=f(x_t)+J_f(x_t)(x_{\mathrm{mid}}-x_t)+R_2$ with
$\lVert R_2\rVert\le\tfrac{M}{2}\lVert x_{\mathrm{mid}}-x_t\rVert^2$ and
$M=\sup\lVert\partial^2 f\rVert$. As $x_{\mathrm{mid}}-x_t=\tfrac{\alpha}{2}f(x_t)$,
$R_2=\mathcal{O}(\alpha^2)$, so
$\Psi_\alpha(x_t)=x_t+\alpha f(x_t)+\tfrac{\alpha^2}{2}J_f(x_t)f(x_t)+\mathcal{O}(\alpha^3)$.
Comparing with the exact expansion, the $\mathcal{O}(\alpha^2)$ terms cancel,
giving \eqref{eq:s-heun}~\cite{hairer1993solving}.
\end{proof}

\noindent\textbf{The step size alone buys a factor of four.} Setting
$\alpha=0.5$ (vs.\ the ResNet default
$\alpha=1$~\cite{he2016deep}) reduces the Euler LTE coefficient by
$4\times$ at no extra FLOP cost; the midpoint step at Stage~3 removes that leading term altogether, taking the local error from $\tfrac{\alpha^2}{2}\lVert J_ff\rVert$ at $\alpha=0.5$ to a residual of order $\alpha^3$.

\begin{figure}[t]
  \centering
  \includegraphics[width=\linewidth]{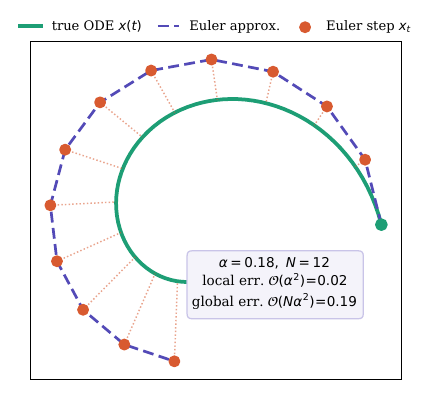}
  \caption{\textbf{Why a smaller step and one second-order stage: the
  motivation, on a model problem.} A forward-Euler discretization of a damped
  spiral $\dot x = Ax$ leaves the true trajectory $x(t)$ by a per-step error of
  order $\mathcal{O}(\alpha^2)$ (dotted segments) that accumulates to
  $\mathcal{O}(N\alpha^2)$ over $N$ steps (Prop.~\ref{supp:p1}). Halving $\alpha$ quarters
  the leading term at no extra cost, and replacing one Euler step with a
  two-stage RK-2 step removes it entirely up to $\mathcal{O}(\alpha^3)$; LIMODENet
  therefore uses $\alpha\!\le\!0.7$ at every stage and one RK-2 stage at the coarsest resolution (Thm.~\ref{supp:t1}). This is a schematic of the numerical argument,
  not a measurement of the trained network.}
  \label{fig:euler-motivation}
\end{figure}

\subsection{Lemma 1: Lipschitz Continuity and Well-Posedness}
\begin{lemma}[Well-posedness]
\label{supp:l1}
Let $f(x;\theta)=W_2\cdot\GELU(W_{\mathrm{dw}}\ast\GN(x))+b$ and suppose training
enforces weight decay $\lambda\lVert W\rVert_F^2$ with $\lambda>0$. Then on any
compact set $K\subset\R^n$ there is $L(\theta,K)<\infty$ with
$\lVert f(x)-f(y)\rVert\le L(\theta,K)\lVert x-y\rVert$ for $x,y\in K$, and by
Picard--Lindel\"of~\cite{coddington1955theory} the IVP has a unique solution on
any compact interval.
\end{lemma}

\begin{proof}
We bound each component and use sub-multiplicativity of Lipschitz constants under
composition. \emph{(GN)}~\cite{wu2018group} $\GN$ is \emph{locally} Lipschitz:
on compact $K$ there is $L_{\GN}(K)<\infty$ depending on $B_\gamma$, $\varepsilon$
and $\sup_{x\in K}\sigma(x)$. Under weight decay and gradient clipping activations
stay bounded, so trajectories remain in a compact set and $L_{\GN}$ is finite.
\emph{(Conv)} A convolution is linear with $\Lip(W\ast\,)=\lVert W\rVert_{\mathrm{op}}\le\lVert W\rVert_F$;
weight decay $\lambda\lVert W\rVert_F^2\le C_{\mathrm{train}}$ gives
$\lVert W\rVert_{\mathrm{op}}\le\sqrt{C_{\mathrm{train}}/\lambda}<\infty$~\cite{miyato2018spectral}.
\emph{(GELU)}~\cite{hendrycks2016gaussian} $\GELU'(t)=\Phi(t)+t\phi(t)$ attains
its supremum at $t\approx1.324$, so $\Lip(\GELU)\le1.129$.
\emph{(Composition)} $\Lip(f)|_K\le\lVert W_2\rVert_{\mathrm{op}}\cdot1.129\cdot\lVert W_{\mathrm{dw}}\rVert_{\mathrm{op}}\cdot L_{\GN}(K)=:L(\theta,K)$.
Picard--Lindel\"of~\cite{coddington1955theory} then yields a unique
$C^1$ solution.
\end{proof}

\noindent\textbf{Two consequences follow.} (a) Weight decay is a well-posedness
requirement, not a tuning knob, because without it
$\lVert W\rVert_{\mathrm{op}}$ may diverge. (b)
BatchNorm~\cite{ioffe2015batch} makes $f$ depend on the batch and so violates
the autonomy of the IVP, which is why LIMODENet uses GroupNorm.

\subsection{Theorem 1: Pareto Placement of the RK-2 Step}
\begin{theorem}[Pareto placement]
\label{supp:t1}
Consider a $K$-stage hierarchy with resolutions $H_k\times H_k$ and shared width
$C$. Assume \textbf{(A)} each $f_k$ satisfies Lemma~1 with $\lVert J_{f_k}f_k\rVert$
uniformly bounded, and \textbf{(B)} $\Cost(k)\propto H_k^2 C$~\cite{molchanov2017pruning}.
Then replacing Euler by the midpoint step improves the LTE order from $\mathcal{O}(\alpha^2)$
to $\mathcal{O}(\alpha^3)$ (a stage-independent gain) at cost $\propto H_k^2 C$,
so the Pareto-optimal placement is a stage of minimal $H_k$.
\end{theorem}

\begin{proof}
By Proposition~\ref{supp:p1},
$\tau_{\mathrm{Euler}}(k)=\tfrac{\alpha^2}{2}\lVert J_{f_k}f_k\rVert+\mathcal{O}(\alpha^3)$
and $\tau_{\mathrm{RK2}}(k)=\mathcal{O}(\alpha^3)$; by (A) the reduction
$\Delta\tau(k)$ is bounded by stage-independent constants of order $\alpha^2$. By
(B) the extra cost of the midpoint step is one evaluation, $\Delta\mathrm{FLOPs}(k)\propto H_k^2 C$.
Hence the efficiency ratio $E(k)=\Delta\tau(k)/\Delta\mathrm{FLOPs}(k)\propto\alpha^2 M/(H_k^2 C)$
is decreasing in $H_k$ and maximized at any stage of minimal resolution.
\end{proof}

\noindent\textbf{The theorem fixes the resolution; a secondary argument fixes the stage.} In LIMODENet $H_1=32$ and $H_2=H_3=16$, so the midpoint step at the coarse resolution attains the same first-order LTE reduction at $(32/16)^2=4\times$ lower cost than at Stage~1. The theorem is indifferent between Stages~2 and~3, which share that resolution; we place the step in Stage~3 because it is a single block at the end of the flow, where accumulated error meets the head, whereas Stage~2 is a four-block chain in which every block would need the extra evaluation. This is the design we implement, obtained by argument rather than ablation. A
weaker form drops (A): Stage~3 stays optimal unless $\lVert J_{f_3}f_3\rVert$ exceeds
$\lVert J_{f_1}f_1\rVert$ by more than $4\times$.

\subsection{Proposition 2: FocalBlock as a Mean-Field ODE}
\looseness=-1 \begin{proposition}[Mean-field coupling]
\label{supp:p2}
Identify each spatial position $(i,j)\in\Omega$ with a particle of state
$h_{:,i,j}\in\R^C$ and let $\loc_{:,i,j}=W_{\mathrm{dw}}\ast h|_{(i,j)}$. With the
empirical measure $\mu_h=\tfrac{1}{|\Omega|}\sum_{(i,j)}\loc_{:,i,j}$, the
FocalBlock field
$f_{\mathrm{focal}}(h)_{:,i,j}=\PWMLP(\GELU(\loc_{:,i,j}+W_g\mu_h))$
is a spatially discretized McKean--Vlasov (mean-field)
ODE~\cite{mckean1966class,sznitman1991topics}; one Euler step is its
propagation-of-chaos discretization.
\end{proposition}

\begin{proof}
A McKean--Vlasov system reads $\dot{x}_i=F(x_i,\mu_t)$,
$\mu_t=\tfrac1N\sum_j\delta_{x_j(t)}$~\cite{sznitman1991topics}. The local branch
is a per-particle map; AdaptiveAvgPool computes
$\tfrac1{|\Omega|}\sum_{(i,j)}\loc_{:,i,j}=\mu_h$; the $1\times1$ conv $W_g$
applies a learned linear transform broadcast to every position. Thus
$f_{\mathrm{focal}}(h)_{:,i,j}=F(h_{:,i,j},\mu_h)$ with $F(u,m)=\PWMLP(\GELU(u+W_gm))$,
a McKean--Vlasov right-hand side with $N=|\Omega|=HW$ particles; the four Stage-2
blocks are four Euler steps with $\alpha=0.7$.
\end{proof}

\noindent\textbf{Three consequences follow.} (a) The coupling $W_g\mu_h$ is permutation
invariant, giving the FocalBlock a partially permutation-invariant inductive bias
suited to scene-level classification. (b) The mean-field summary is $\mathcal{O}(HW)$
and uses only AvgPool$+$Conv$1{\times}1$, both on the Akida/Loihi-2 operation
lists, unlike pairwise attention~\cite{vaswani2017attention}. (c) The construction
connects to infinite-width mean-field
theory~\cite{mei2018mean,chizat2018global}; the CNN extension is left as future
work.

\begin{figure*}[t]
  \centering
  \includegraphics[width=\linewidth]{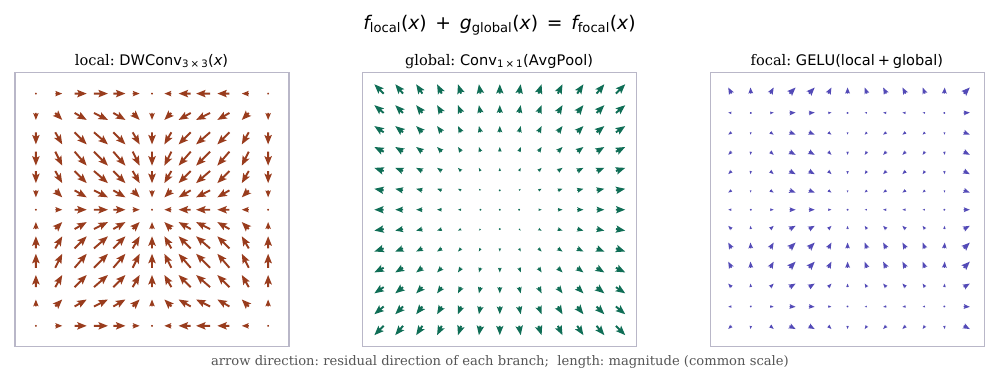}
  \caption{\textbf{What the FocalBlock right-hand side is built to do, and the
  expected effect.} The block field is the GELU-gated sum of a local branch
  ($\mathrm{DWConv}_{3\times3}$, which by itself drives features only toward neighborhood structure, with no long-range term) and a global branch
  ($\mathrm{Conv}_{1\times1}$ on a spatial average, a single position-independent
  vector broadcast everywhere, i.e.\ the first moment $\mu_h$ of Prop.~\ref{supp:p2}). Adding
  the mean-field term makes every output position depend on every input in one step (Prop.~\ref{supp:p6}), which pure stacked $3\times3$ convolutions reach only after
  ${\ge}15$ layers at this resolution; the additive (not multiplicative) form
  keeps the combined field bounded and gradient-friendly. The panels plot
  analytic branch models on a $2$-D cartoon domain to isolate the construction;
  they are not activations of a trained model.}
  \label{fig:focal-field}
\end{figure*}

\begin{figure*}[t]
  \centering
  \includegraphics[width=\linewidth]{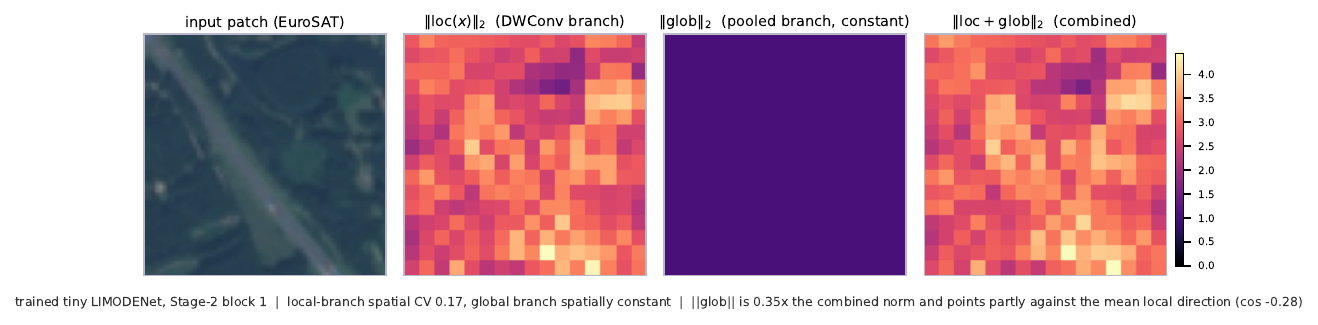}
  \caption{\textbf{The same local/global split, measured on a trained model.}
  Per-position channel-$\ell_2$ norm of the two branches of the first Stage-2
  FocalBlock of a trained \emph{tiny} LIMODENet (EuroSAT, Tiny-ImageNet
  pretraining then fine-tuning, seed~0), on one held-out EuroSAT patch; panels
  2--4 share a color scale. The local (DWConv) branch carries spatially varying
  structure (coefficient of variation $0.17$); the global branch is
  \emph{exactly} spatially constant, a single pooled vector $W_g\mu_h$ broadcast to every position (Prop.~\ref{supp:p2}), with norm $0.35\times$ the combined
  field, and it points partly against the mean local direction
  ($\cos=-0.28$), i.e.\ it acts as a scene-level corrective offset rather than a
  copy of the local response. This is the behavior
  Fig.~\ref{fig:focal-field} sketches, now on real activations.}
  \label{fig:focal-branches}
\end{figure*}

\subsection{Proposition 3: Contractivity under Dissipativity}
\begin{proposition}[Contractivity]
\label{supp:p3}
Let $f$ satisfy Lemma~1 with constant $L$ and the one-sided Lipschitz
(dissipativity) condition
$\langle f(x)-f(y),x-y\rangle\le-m\lVert x-y\rVert^2$ with $m>0$. Then the Euler
map $T(x)=x+\alpha f(x)$ obeys
$\lVert T(x)-T(y)\rVert\le\kappa\lVert x-y\rVert$ with
$\kappa=\sqrt{1-2\alpha m+\alpha^2L^2}$, and $\kappa<1$ iff $\alpha<2m/L^2$; the
iteration then converges geometrically to a unique fixed
point~\cite{soderlind2006logarithmic,hairer1993solving}.
\end{proposition}

\begin{proof}
$\lVert T(x)-T(y)\rVert^2=\lVert x-y\rVert^2+2\alpha\langle f(x)-f(y),x-y\rangle+\alpha^2\lVert f(x)-f(y)\rVert^2$.
Applying dissipativity to the middle term and $\lVert f(x)-f(y)\rVert^2\le L^2\lVert x-y\rVert^2$
to the last gives $\lVert T(x)-T(y)\rVert^2\le(1-2\alpha m+\alpha^2L^2)\lVert x-y\rVert^2$.
Setting $\kappa^2=1-2\alpha m+\alpha^2L^2$, $\kappa<1\Leftrightarrow\alpha<2m/L^2$;
Banach's theorem gives the fixed point.
\end{proof}

\noindent\textbf{This consequence remains heuristic.} Dissipativity is \emph{not} enforced in
training. Cross-entropy plausibly induces an implicit pressure toward it, but
whether the strict inequality holds along trajectories is an open empirical
question, to be checked by numerically estimating the one-sided Lipschitz constant.
We therefore treat Proposition~\ref{supp:p3} as a finite-step stability statement rather than a
fixed-point-selection claim, avoiding tension with the injectivity role of
Proposition~\ref{supp:p5}. GroupNorm, weight decay (widening $\alpha<2m/L^2$), and
cosine-annealed learning rates all encourage the discrete iteration to stay inside
its stability region.

\subsection{Proposition 4: The Composite Field $g=f\circ\GN$}
\begin{proposition}[Composite field]
\label{supp:p4}
Let $g(x;\theta):=f(\GN(x);\theta)$. \emph{(i)} LIMODENet's Euler pass is
$T(x)=x+\alpha g(x)$ and its RK-2 pass is exactly the explicit midpoint method applied to $g$:
$k_1=g(x_t)$, $x_{\mathrm{mid}}=x_t+\tfrac{\alpha}{2}k_1$, $k_2=g(x_{\mathrm{mid}})$,
$x_{t+1}=x_t+\alpha k_2$. \emph{(ii)} All conclusions of
Propositions~\ref{supp:p1}--\ref{supp:p3} and Theorem~\ref{supp:t1} transfer to
$g$, with $\Lip(g)|_K\le\Lip(f)|_{\GN(K)}\,L_{\GN}(K)$.
\end{proposition}

\begin{proof}
\emph{(i)} We match the forward pass line by line: \texttt{h=norm(x)}$\Rightarrow
\GN(x_t)$; \texttt{k1=mlp(act(dw(h)))}$\Rightarrow g(x_t)$;
\texttt{mid=x+(alpha/2)k1}; \texttt{k2=mlp(act(dw(norm(mid))))}$\Rightarrow g(x_{\mathrm{mid}})$;
\texttt{return x+alpha*k2}. This is the midpoint method on $g$; the Euler branch computes
$x+\alpha g(x)$. \emph{(ii)} Composition of Lipschitz maps is Lipschitz, giving the
stated bound (finite by Lemma~\ref{supp:l1}); $g\in C^2$ so the LTE estimates apply
verbatim; Proposition~\ref{supp:p3} transfers with $(m_g,L_g)$; Theorem~\ref{supp:t1} depends only on
stage-wise smoothness and cost.
\end{proof}

\noindent\textbf{The step sizes belong to the composite ODE.} They are step
sizes for $\dot x=g(x)$, so empirical estimates of $L_g,m_g$ must include
GroupNorm in the Jacobian. This correction replaces an earlier (incorrect) claim
that GroupNorm at the midpoint evaluation breaks the RK-2 step.

\subsection{Proposition 5: Injectivity Preserves Semantic Information}
\begin{proposition}[Injectivity $\Rightarrow$ information preservation]
\label{supp:p5}
Let $g=f\circ\GN$ have local Lipschitz constant $L_g$. \emph{(i)} If
$\alpha L_g<1$ the Euler map $T(x)=x+\alpha g(x)$ is
injective~\cite{behrmann2019invertible}. \emph{(ii)} The midpoint map is injective whenever
\begin{equation}
  \alpha L_g\big(1+\tfrac{\alpha}{2}L_g\big)<1.
  \label{eq:s-heuninj}
\end{equation}
\emph{(iii)} If $T$ is injective then for any random variable $S$ (in particular
the label $Y$),
\begin{equation}
  I(S;T(X))=I(S;X).
  \label{eq:s-dpi}
\end{equation}
\end{proposition}

\looseness=-1 \begin{proof}
\emph{(i)} If $T(x)=T(y)$ with $x\ne y$ then $x-y=-\alpha(g(x)-g(y))$, so
$\lVert x-y\rVert=\alpha\lVert g(x)-g(y)\rVert\le\alpha L_g\lVert x-y\rVert<\lVert x-y\rVert$,
a contradiction~\cite{behrmann2019invertible}. \emph{(ii)} Write $T(x)=x+\alpha g(M(x))$
with $M(x)=x+\tfrac{\alpha}{2}g(x)$; then
$\Lip(g\circ M)\le L_g(1+\tfrac{\alpha}{2}L_g)$, and applying (i) to the residual
$\alpha g(M(\cdot))$ gives \eqref{eq:s-heuninj}. \emph{(iii)} The data-processing
inequality~\cite{cover2006elements} gives $I(S;T(X))\le I(S;X)$; since $T$ is
injective, $X=T^{-1}(T(X))$ is a deterministic function of $T(X)$, so the reverse
DPI gives $I(S;X)\le I(S;T(X))$, and equality \eqref{eq:s-dpi} follows: no label information is lost.
\end{proof}

\looseness=-1 \noindent\textbf{This has three consequences.} (a) The chosen step sizes widen the margin over a
ResNet ($\alpha=1$, needing $L_g<1$): Stage~1 ($\alpha{=}0.5$) is injective iff
$L_g<2.0$, Stage~2 ($\alpha{=}0.7$) iff $L_g<1.43$, Stage~3 (midpoint, $\alpha{=}0.5$,
solving \eqref{eq:s-heuninj}) iff $L_g<1.46$. (b) A dimension audit shows that
the stem
\emph{expands} $3{\times}64{\times}64=12{,}288$ to $128{\times}32{\times}32=131{,}072$
($10.7\times$); the only compressive backbone op is the single downsample
($131{,}072\!\to\!32{,}768$) and the only deliberate semantic compression is the
head's pool ($32{,}768\!\to\!128$). Hence $I(Y;x_\ell)$ is conserved through every
ODE block and reduced only at two chosen points, i.e.\ a \emph{late-compression} architecture, consistent with i-RevNet~\cite{jacobsen2018irevnet}. (c) Injectivity
concerns Shannon information, not \emph{linear accessibility}; we therefore pair
the claim with layer-wise probing (\cref{sec:exp}). A per-block Lipschitz estimate
via power iteration on Jacobian--vector products converts the conditional
guarantee into a verified property.

\subsection{Proposition 6: Global Receptive Field at Sub-Attention Cost}
\begin{proposition}[Global RF, sub-attention cost]
\label{supp:p6}
Let a feature map have $N=HW$ positions and $C$ channels. \emph{(i)} After a single
FocalBlock every output position depends on every input position, whereas a
$k\times k$ convolution stack needs $\ell\ge2(H-1)/(k-1)$ layers for the same
coverage~\cite{luo2016understanding} ($15$ layers at $H{=}16,k{=}3$). \emph{(ii)}
Its token-mixing cost is
$\Cost_{\mathrm{focal}}=\mathcal{O}(NCk^2+NC+C^2)$, linear in $N$, versus
$\Cost_{\mathrm{MHSA}}=\mathcal{O}(N^2C+NC^2)$~\cite{vaswani2017attention}.
\end{proposition}

\begin{proof}
\emph{(i)} $\mathrm{glob}=W_g\mu(\loc)$ with $\mu(\loc)=\tfrac1N\sum_{(i,j)}\loc_{:,i,j}$
sums over all positions; the output
$\PWMLP(\GELU(\loc_{:,i,j}+\mathrm{glob}))$ therefore depends on every input, the
Jacobian $\partial\,\mathrm{out}_{(i,j)}/\partial\,\mathrm{in}_{(i',j')}$ containing
the generically nonzero term $J_{\PWMLP}\!\cdot\GELU'\!\cdot W_g\!\cdot\tfrac1N\!\cdot w_{\mathrm{dw}}$.
For a pure stack, each layer extends the Chebyshev dependence radius by
$(k-1)/2$, so full corner-to-corner coverage needs $\ell\ge2(H-1)/(k-1)$
layers~\cite{luo2016understanding}. \emph{(ii)} DWConv costs $NCk^2$, AvgPool $NC$,
and the pooled $1\times1$ conv $C^2$; MHSA costs $2N^2C$ (for $QK^\top$ and its
application to $V$) plus $4NC^2$ for projections. At $N{=}256$, $C{=}128$ and
$k{=}3$ these evaluate to
$\Cost_{\mathrm{focal}}\approx3.5\times10^5$ versus
$\Cost_{\mathrm{MHSA}}\approx3.4\times10^7$, a reduction of about $96\times$.
\end{proof}

\looseness=-1 \noindent\textbf{We draw three consequences.} (a) One FocalBlock removes CNN locality: four
sequential blocks give four local$\leftrightarrow$global exchanges that a $3\times3$
CNN could not achieve in four layers ($\ge15$ needed). (b) The $\approx96\times$ saving grows with resolution (${\approx}250\times$ at $32\times32$, where the quadratic attention term dominates but the projection term does not), and all
ops are Akida/Loihi-2 native. (c) One honest limit remains: the mean-field
coupling supplies only the first moment of the feature distribution, which is weaker than pairwise attention for the dense-prediction tasks we do not target here.

%%%%%%%%%%%%%%%%%%%%%%%%%%%%%%%%%%%%%%%%%%%%%%%%%%%%%%%%%%%%%%%%%%%%%%%%
\section{Placement in the Main Paper}
\label{supp:placement}

Proposition~\ref{supp:p1} (LTE), Lemma~\ref{supp:l1} (well-posedness) and
Theorem~\ref{supp:t1} (Pareto placement) underpin the integrator choices in
\cref{sec:method}; Proposition~\ref{supp:p5} (injectivity) and
Proposition~\ref{supp:p6} (global RF) are stated there as the two load-bearing
results and are the design's core claims. Propositions~\ref{supp:p2}--\ref{supp:p4}
provide the mean-field interpretation, the finite-step stability statement, and the
composite-field bookkeeping that makes the other results apply to the
\emph{implemented} network. The injectivity margins of
Proposition~\ref{supp:p5}(a) and the Lipschitz estimates they require are verified
empirically in \cref{sec:exp}.

%%%%%%%%%%%%%%%%%%%%%%%%%%%%%%%%%%%%%%%%%%%%%%%%%%%%%%%%%%%%%%%%%%%%%%%%
\section{Above the Decoding Threshold: Two Null Results}
\label{supp:abovecliff}

\Cref{sec:recon-compare} restricts the encoder comparison to received signals below
the link's decoding threshold. This appendix reports the two settings above it,
$3$ and $4$\,dB, in full. Neither supports a claim, and we include them so that the
operating envelope is documented rather than implied.

\paragraph{Aggregate results.}
Table~\ref{supp:tab:abovecliff} gives the same three-seed, iso-parameter,
$60$-epoch protocol used throughout, scored with the neutral Spiking-CNN judge.
At $3$\,dB the three architectures fall within $0.38$\,dB and $0.12$\,pp of one
another; at $4$\,dB, within $0.05$\,dB and $0.19$\,pp. In both settings the ranking
is unstable: the U-Net is nominally first on PSNR at $3$\,dB and on Top-1 at $4$\,dB, and of the two bottleneck models the CNN-AE is nominally ahead on Top-1 at $4$\,dB.

\paragraph{Why ``tie'' needs a formal test, and why the tolerance must be
calibrated independently.} Non-overlapping error bars can show a difference;
overlapping bars cannot show a tie: the absence of a detected difference is not
evidence of equivalence, only of an underpowered test or a genuinely small effect,
and the two are indistinguishable from the bars alone. We test equivalence
directly with two one-sided tests
(TOST)~\cite{schuirmann1987comparison,lakens2017equivalence}: fixing a tolerance
$\delta_0$, a comparison is declared equivalent only when the $90\%$ confidence
interval of the mean difference falls entirely inside $[-\delta_0,\delta_0]$.

\looseness=-1 $\delta_0$ must be fixed independently of the comparisons it will be applied to, or
the test is circular. We calibrate it against two comparisons from the ablation
sweep (\cref{sec:ablations}) whose separation status was established for a different,
theoretical reason before this exercise: a \emph{positive} control, the learned-$\alpha$ variant versus the reference (the ablation's own finding is that any sub-unit $\alpha$ suffices, i.e.\ a negligible difference is expected on independent grounds, observed diff $0.04$\,dB/$0.04$\,pp), and a \emph{negative} control, the spectral-normalized variant versus the reference (the ablation's largest, most robust effect, observed diff $3.22$\,dB/$4.97$\,pp). Setting
$\delta_0$ a quarter of the way from the positive- to the negative-control
magnitude gives $\delta_0{=}0.84$\,dB / $1.27$\,pp, and we verify by construction that this bound does \emph{not} call the negative control equivalent, so it retains the power to detect a real effect of that magnitude.

Under this calibrated bound, all four above-threshold comparisons in
Table~\ref{supp:tab:abovecliff} are confirmed equivalent, including
LIMODENet-versus-CNN-AE at $3$\,dB ($+0.35$\,dB, $90\%$ CI
$[+0.25,+0.46]$\,dB, inside $\pm0.84$). We flag this because an earlier pass of
this analysis, using an uncalibrated $\delta_0{=}0.2$\,dB chosen by eye from the
smallest margin this paper calls separating elsewhere, reported that same
comparison as a confirmed \emph{separation}. It was not: a plausible-looking but
uncalibrated tolerance manufactured a false positive by being narrower than the
comparison's own noise floor. We keep this correction visible rather than silently
fixing it, since it is itself evidence for the point being made: an
equivalence claim is only as good as the procedure that set its tolerance, and
``half the
smallest separating margin'' is not such a procedure. \textbf{All four above-cliff
comparisons stand as ties under a formally calibrated test}, unlike the
downstream comparison in \cref{sec:recon-compare}, which the same procedure leaves
\emph{inconclusive} rather than confirmed either way.

\begin{table}[h]
  \centering
  \footnotesize
  \setlength{\tabcolsep}{4pt}
  \begin{tabular}{llccc}
    \toprule
    & Reconstructor & PSNR (dB) & SSIM & Top-1 (\%) \\
    \midrule
    \multirow{3}{*}{\rotatebox{90}{$3$\,dB}}
      & LIMODENet-AE & 46.42{\tiny$\pm$0.01} & 0.9974{\tiny$\pm$0.0000} & 87.60{\tiny$\pm$0.09} \\
      & U-Net        & 46.45{\tiny$\pm$0.05} & 0.9975{\tiny$\pm$0.0001} & 87.66{\tiny$\pm$0.10} \\
      & CNN-AE       & 46.07{\tiny$\pm$0.06} & 0.9969{\tiny$\pm$0.0001} & 87.54{\tiny$\pm$0.05} \\
    \midrule
    \multirow{3}{*}{\rotatebox{90}{$4$\,dB}}
      & LIMODENet-AE & 39.60{\tiny$\pm$0.00} & 0.9966{\tiny$\pm$0.0000} & 87.31{\tiny$\pm$0.06} \\
      & U-Net        & 39.60{\tiny$\pm$0.00} & 0.9967{\tiny$\pm$0.0001} & 87.50{\tiny$\pm$0.07} \\
      & CNN-AE       & 39.55{\tiny$\pm$0.00} & 0.9963{\tiny$\pm$0.0001} & 87.45{\tiny$\pm$0.08} \\
    \bottomrule
  \end{tabular}
  \caption{\textbf{Above the decoding threshold no architecture leads.} Three
  seeds, iso-parameter, $60$ epochs, val-selected, neutral judge. A formal
  equivalence test calibrated against independent controls (TOST, see below)
  confirms all four comparisons (LIMODENet versus the U-Net and versus the CNN-AE, at both $3$ and $4$\,dB) as ties. Compare with the sub-threshold margins of up
  to $+1.53$\,dB and $+9.8$\,pp in \cref{supp:tab:snr,supp:tab:quality}.}
  \label{supp:tab:abovecliff}
\end{table}

\paragraph{Why the $4$\,dB aggregate is not a meaningful quantity.}
\looseness=-1 Both settings retain a small tail of frames the decoder failed to deliver: $8$ of
$8090$ test images at $3$\,dB ($0.10\%$) and $16$ at $4$\,dB ($0.20\%$). Because we
report PSNR from the summed squared error over the whole test set, these few frames
dominate it. At $4$\,dB the $16$ failed frames account for $93.9\%$ of the total
squared error, so the reported $39.6$\,dB is a measurement of those $16$ images
rather than of restoration quality on the other $8074$. This also explains the
otherwise puzzling ordering that $4$\,dB scores $7$\,dB \emph{below} $3$\,dB despite
being the milder link, and the unusually small seed spread at $4$\,dB ($\pm0.002$\,dB), since most of the metric is pinned by frames that no model alters.

\paragraph{Stratified results, and a recoverability threshold.}
Splitting the test set by received quality (Table~\ref{supp:tab:strat}) separates
the two effects. At $3$\,dB the failed frames arrive at $22.1$\,dB and are partly
recoverable: LIMODENet and the U-Net both restore them by ${\approx}{+}7.2$\,dB and
the CNN-AE by $+4.2$. At $4$\,dB the failed frames arrive at $14.0$\,dB and are
recovered by \emph{exactly nothing}: all three architectures reproduce the input to
two decimal places. Somewhere between $14$ and $22$\,dB of received quality, a frame
stops being restorable at all, which is a property of the residual information in
the signal rather than of any architecture.

\begin{table}[h]
  \centering
  \footnotesize
  \setlength{\tabcolsep}{1.5pt}
  \begin{tabular}{llcccc}
    \toprule
    & & \multicolumn{2}{c}{failed frames} & \multicolumn{2}{c}{remaining frames} \\
    \cmidrule(lr){3-4}\cmidrule(lr){5-6}
    & Reconstructor & in$\to$out & gain & in$\to$out & gain \\
    \midrule
    \multirow{3}{*}{\rotatebox{90}{$3$\,dB}}
      & LIMODENet-AE & $22.10\to29.25$ & $+7.15$ & $50.48\to52.25$ & $+1.77$ \\
      & U-Net        & $22.10\to29.27$ & $+7.16$ & $50.48\to52.39$ & $+1.91$ \\
      & CNN-AE       & $22.10\to26.34$ & $+4.24$ & $50.48\to51.55$ & $+1.07$ \\
    \midrule
    \multirow{3}{*}{\rotatebox{90}{$4$\,dB}}
      & LIMODENet-AE & $14.04\to14.04$ & $-0.00$ & $50.78\to52.11$ & $+1.33$ \\
      & U-Net        & $14.04\to14.04$ & $-0.00$ & $50.78\to52.08$ & $+1.31$ \\
      & CNN-AE       & $14.04\to14.04$ & $-0.00$ & $50.78\to51.51$ & $+0.74$ \\
    \bottomrule
  \end{tabular}
  \caption{\looseness=-1 \textbf{Stratified by received quality: input\,$\to$\,output PSNR (dB),
  seed~0.} ``Failed frames'' are received images below $30$\,dB
  ($8$ of $8090$ at $3$\,dB, $16$ at $4$\,dB); the remainder arrive essentially
  lossless. At $4$\,dB the failed frames are not recovered at all by any of the three architectures.}
  \label{supp:tab:strat}
\end{table}

\paragraph{What these nulls indicate.}
Read together with \cref{supp:tab:quality,supp:tab:snr}, the pattern is consistent rather than
merely negative. LIMODENet's advantage over the \emph{skip-connection} U-Net appears
only below the decoding threshold ($+1.07$\,dB at $1$\,dB and $+0.73$ at $2$\,dB) and vanishes above it, at both $3$ and $4$\,dB and on both strata. That is
what the bottleneck argument of \cref{sec:recon-compare} predicts: routing
high-resolution detail around a bottleneck costs nothing when the input retains that
detail, so the two designs should coincide on a near-lossless link, and should
diverge only when information has actually been destroyed and must be preserved
through a single path. The advantage over the \emph{same-class} CNN-AE, by contrast,
persists in reduced form above the threshold ($+1.33$ versus $+0.74$\,dB on
unfailed frames at $4$\,dB), consistent with it being a capacity-and-design
difference within the bottleneck family rather than a bottleneck-versus-skip
difference. We report these as observations, not as claims: the effects above the
threshold are small, and the failed-frame strata contain only $8$ and $16$ images.

\section{Per-Block Injectivity Measurements}
\label{supp:injectivity}

\looseness=-1 The layer-wise probing curve that establishes empirical information preservation is Fig.~\ref{fig:probing} in the main paper; this section gives the per-block Lipschitz constants and inversion tests that accompany it.

\begin{figure*}[t]
  \centering
  \includegraphics[width=\linewidth]{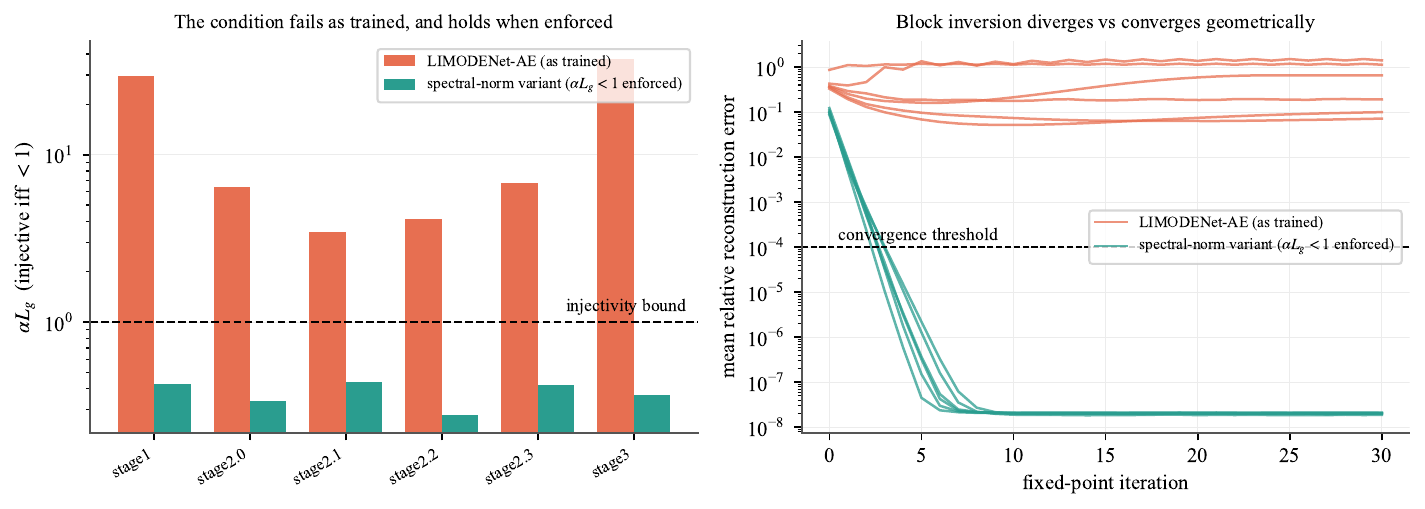}
  \caption{\textbf{The injectivity condition fails everywhere in the trained
  model, holds everywhere when enforced, and costs $3.22$\,dB to enforce.}
  \emph{Left:} per-block $\alpha L_g$, estimated by JVP/VJP power iteration at
  real operating points ($8$ points per block, $22$ iterations); the axis is
  logarithmic because the trained values exceed the bound by one to two orders of
  magnitude. \emph{Right:} i-ResNet fixed-point inversion of each block. The
  trained model's error is flat or growing, because the iteration is not a
  contraction and so has no fixed point to reach, whereas the constrained
  variant falls
  geometrically to machine precision in every block.}
  \label{fig:injectivity}
\end{figure*}

\looseness=-1 \cref{fig:injectivity} summarises the injectivity analysis; the per-block values
behind it are given in \cref{tab:injectivity-full}. Lipschitz constants are the
maximum over $8$ operating points captured from real degraded inputs at
$1$\,dB/$q100$, each estimated by $22$ steps of power iteration on
Jacobian--vector products of $g=f\circ\GN$. ``Inverts'' records whether the
i-ResNet fixed-point iteration $x_{k+1}=y-\alpha\,r(x_k)$ reaches a mean relative
error below $10^{-4}$ within $30$ iterations, using each block's true forward map,
so the test is exact for the Euler, focal and midpoint blocks alike, with no linearization.

Two features of the trained model are worth noting. Stage~1 is the worst
violator, which follows from the input variance being smallest there and
GroupNorm dividing by it; and the midpoint stage has the largest condition value
despite a middling $L_g$, because its condition
$\alpha L_g\,(1+\tfrac{\alpha}{2}L_g)$ is quadratic in $L_g$. Neither is an
artifact of the estimator: the inversion column is an independent test that makes
no reference to any bound, and it agrees with the condition in every one of the
twelve cases.

\begin{table}[h]
  \centering
  \footnotesize
  % auto-generated by theory_recon_plot.py -- do not hand-edit
\begin{tabular}{llcccc}
\toprule
Block & Disc. & $\alpha$ & $\hat L_g$ & $\alpha L_g$ & inverts? \\
\midrule
\multicolumn{6}{l}{\emph{LIMODENet-AE (as trained)}} \\
~~stage1 & euler & 0.5 & 58.84 & 29.42 & \ding{55} \\
~~stage2.0 & focal & 0.7 & 9.12 & 6.39 & \ding{55} \\
~~stage2.1 & focal & 0.7 & 4.94 & 3.46 & \ding{55} \\
~~stage2.2 & focal & 0.7 & 5.87 & 4.11 & \ding{55} \\
~~stage2.3 & focal & 0.7 & 9.71 & 6.79 & \ding{55} \\
~~stage3 & midpoint & 0.5 & 15.39 & 37.32 & \ding{55} \\
\midrule
\multicolumn{6}{l}{\emph{spectral-norm variant ($\alpha L_g<1$ enforced)}} \\
~~stage1 & euler & 0.5 & 0.85 & 0.43 & \checkmark \\
~~stage2.0 & focal & 0.7 & 0.48 & 0.34 & \checkmark \\
~~stage2.1 & focal & 0.7 & 0.63 & 0.44 & \checkmark \\
~~stage2.2 & focal & 0.7 & 0.40 & 0.28 & \checkmark \\
~~stage2.3 & focal & 0.7 & 0.60 & 0.42 & \checkmark \\
~~stage3 & midpoint & 0.5 & 0.63 & 0.37 & \checkmark \\
\bottomrule
\end{tabular}

  \caption{Per-block injectivity measurements for the trained reconstructor and
  for the spectral-normalized variant. $\hat L_g$ is the estimated local Lipschitz
  constant of the residual field; the condition is $\alpha L_g<1$ for Euler and
  focal blocks and $\alpha L_g(1+\tfrac{\alpha}{2}L_g)<1$ for the midpoint block.
  Auto-generated by \texttt{theory/theory\_recon\_plot.py}.}
  \label{tab:injectivity-full}
\end{table}

\section{Full Restoration Sweep: Quality and SNR Axes}
\label{supp:sweep}

\cref{sec:recon-compare} summarizes the sweep; this section gives the full
tables. A single operating point cannot establish that an architectural
advantage is general, so we repeat the entire three-seed comparison at JPEG
quality $50$ and $10$ as well as $100$, holding $E_s/N_0$ at $1$\,dB
(Table~\ref{supp:tab:quality}). Aggressive compression compounds the channel
and caps what any decoder could recover: reconstruction tops out near
$30.6$\,dB at $10$q against $37.4$ at $100$q. The fidelity ordering is
nevertheless preserved at every setting, with non-overlapping error bars on
both PSNR and SSIM in all six comparisons.

Two regularities are worth stating. First, the margin shrinks
\emph{monotonically} as quality falls: against the U-Net it goes
$+1.07\to+0.49\to+0.24$\,dB, and against the CNN-AE
$+1.75\to+0.99\to+0.37$\,dB. This is the behaviour one should expect rather than
a weakness: once the channel has destroyed the detail, a better encoder has less
left to preserve and the headroom for \emph{any} architecture narrows. Second, on
the semantic metric LIMODENet beats the CNN-AE at all three qualities but the
margin over the U-Net is not resolved at any of the three
($+1.33/+0.82/+0.21$\,pp, every one within overlapping error bars).

\begin{table}[h]
  \centering
  \footnotesize
  \setlength{\tabcolsep}{2.5pt}
  \begin{tabular}{llccc}
    \toprule
    & Reconstructor & PSNR (dB) & SSIM & Top-1 (\%) \\
    \midrule
    \multirow{3}{*}{\rotatebox{90}{$100$q}}
      & \textbf{LIMODENet-AE} & \textbf{37.39}{\tiny$\pm$0.14}
        & \textbf{0.9672}{\tiny$\pm$0.0010} & \textbf{78.50}{\tiny$\pm$0.83} \\
      & U-Net & 36.32{\tiny$\pm$0.33} & 0.9591{\tiny$\pm$0.0027} & 77.17{\tiny$\pm$0.89} \\
      & CNN-AE & 35.64{\tiny$\pm$0.18} & 0.9531{\tiny$\pm$0.0017} & 71.89{\tiny$\pm$0.77} \\
    \midrule
    \multirow{3}{*}{\rotatebox{90}{$50$q}}
      & \textbf{LIMODENet-AE} & \textbf{33.92}{\tiny$\pm$0.05}
        & \textbf{0.9426}{\tiny$\pm$0.0006} & \textbf{67.14}{\tiny$\pm$0.18} \\
      & U-Net & 33.43{\tiny$\pm$0.16} & 0.9358{\tiny$\pm$0.0027} & 66.33{\tiny$\pm$1.15} \\
      & CNN-AE & 32.93{\tiny$\pm$0.05} & 0.9301{\tiny$\pm$0.0006} & 62.78{\tiny$\pm$0.22} \\
    \midrule
    \multirow{3}{*}{\rotatebox{90}{$10$q}}
      & \textbf{LIMODENet-AE} & \textbf{30.56}{\tiny$\pm$0.03}
        & \textbf{0.8865}{\tiny$\pm$0.0004} & \textbf{49.35}{\tiny$\pm$0.59} \\
      & U-Net & 30.32{\tiny$\pm$0.04} & 0.8828{\tiny$\pm$0.0006} & 49.14{\tiny$\pm$0.46} \\
      & CNN-AE & 30.18{\tiny$\pm$0.02} & 0.8813{\tiny$\pm$0.0003} & 47.18{\tiny$\pm$0.10} \\
    \bottomrule
  \end{tabular}
  \caption{\textbf{The fidelity ranking is preserved across the whole quality
  axis.} The same iso-parameter three-seed comparison at $1$\,dB, run at JPEG
  quality $100$, $50$ and $10$; all models trained to convergence ($60$ epochs),
  epochs val-selected. Top-1 uses the neutral Spiking-CNN judge; LIMODENet beats
  the CNN-AE at all three settings but is \emph{tied} with the U-Net at all three,
  which is why the main text scopes the U-Net claim to fidelity.}
  \label{supp:tab:quality}
\end{table}

Varying JPEG quality changes how much of the image survives compression;
varying $E_s/N_0$ changes how much survives the channel itself.
Table~\ref{supp:tab:snr} repeats the comparison at $2$\,dB. The fidelity
ordering is again preserved and separated, by $+0.73$\,dB\,/\,$+0.0086$ SSIM
over the U-Net and $+1.53$\,/\,$+0.0186$ over the CNN-AE. The notable change is
on the semantic metric: at $2$\,dB LIMODENet leads the U-Net by $+3.29$\,pp
with \emph{non-overlapping} error bars, whereas at $1$\,dB the two were tied
at every JPEG quality; the margin over the CNN-AE also widens, from $+6.6$ to
$+9.8$\,pp.

Extending to $3$\,dB (Fig.~\ref{fig:snrgrid}) shows that this emulated DVB-S2X link does not degrade gracefully with $E_s/N_0$; it falls off a cliff between $2$ and $3$\,dB.
Measured on the \emph{received} images before any model runs, over the full
$8{,}090$-image test split and at the $128$\,px working resolution, the median
PSNR against the clean reference is $17.5$\,dB at $1$\,dB and $19.0$\,dB at $2$\,dB, both severely corrupted, but $50.2$\,dB at $3$\,dB and $50.3$\,dB at
$4$\,dB, which is visually lossless. A $31$\,dB change in received quality
across a $1$\,dB change in $E_s/N_0$ is the signature of a coded-link decoding
threshold: DVB-S2X pairs LDPC with an outer BCH code, and such systems
transition sharply from quasi-error-free operation to decoding collapse over a
fraction of a decibel.

Received PSNR depends on the resolution at which it is evaluated. EuroSAT is natively $64\times64$ and the pipeline resamples to $128$, which smooths the
channel noise and raises the $1$\,dB figure from $13.8$ to $17.5$\,dB, so we report it at $128$\,px throughout. The aggregate and per-image statistics
diverge sharply above the threshold: at $4$\,dB the aggregate is $39.6$\,dB against a median of $50.3$, because a small tail of failed frames owns almost
all of the squared error (see \cref{supp:abovecliff} for the full stratified
analysis). Below the threshold, the margin over the CNN-AE runs $+2.2 \to +6.6 \to +9.8$\,pp ($1$\,dB/$10$q, $1$\,dB/$100$q, $2$\,dB/$100$q)
before collapsing to $+0.06$\,pp once the link is clean.

\begin{figure*}[t]
  \centering
  \includegraphics[width=\linewidth]{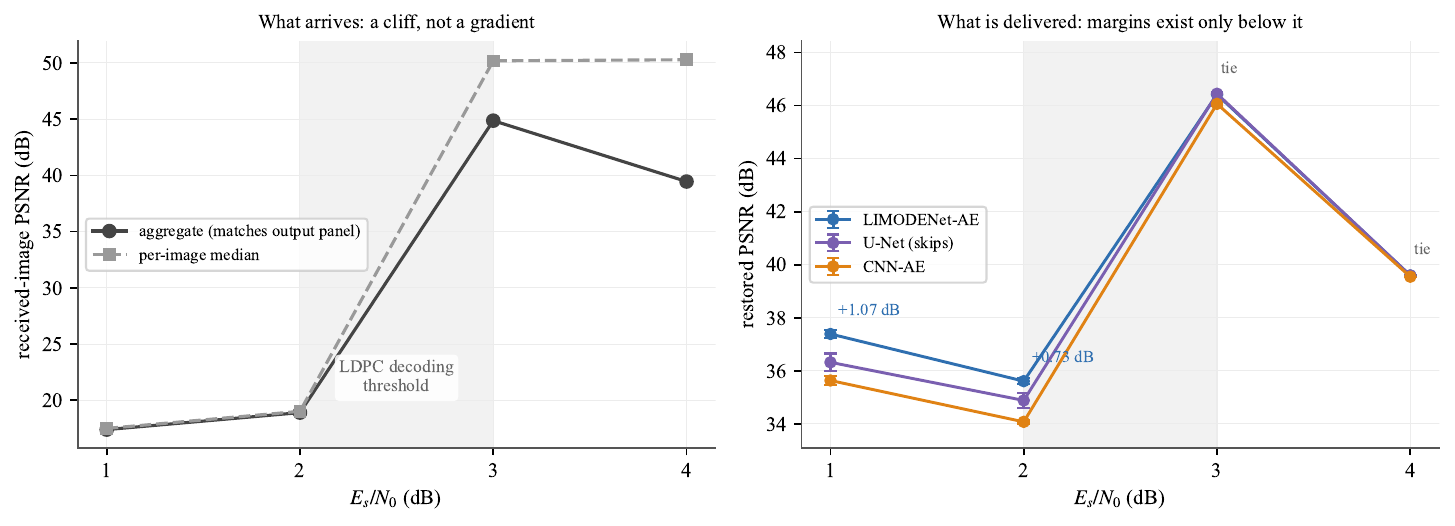}
  \caption{\textbf{The link has a decoding threshold, and the encoder
  comparison lives entirely below it.} \emph{Left:} received-image PSNR before
  any model, measured over the full $8{,}090$-image test split at the
  $128$\,px working resolution. \emph{Right:} restored PSNR for the three
  iso-parameter reconstructors, $3$ seeds, $60$ epochs, val-selected. Margins
  exist only below the threshold and vanish above it, which is what the
  bottleneck reading predicts: routing detail around a bottleneck costs
  nothing when the input still contains that detail.}
  \label{fig:snrgrid}
\end{figure*}

\begin{table}[t]
  \centering
  \footnotesize
  \setlength{\tabcolsep}{2.5pt}
  \begin{tabular}{llccc}
    \toprule
    & Reconstructor & PSNR (dB) & SSIM & Top-1 (\%) \\
    \midrule
    \multirow{3}{*}{\rotatebox{90}{$1$\,dB}}
      & \textbf{LIMODENet-AE} & \textbf{37.39}{\tiny$\pm$0.14}
        & \textbf{0.9672}{\tiny$\pm$0.0010} & \textbf{78.50}{\tiny$\pm$0.83} \\
      & U-Net & 36.32{\tiny$\pm$0.33} & 0.9591{\tiny$\pm$0.0027} & 77.17{\tiny$\pm$0.89} \\
      & CNN-AE & 35.64{\tiny$\pm$0.18} & 0.9531{\tiny$\pm$0.0017} & 71.89{\tiny$\pm$0.77} \\
    \midrule
    \multirow{3}{*}{\rotatebox{90}{$2$\,dB}}
      & \textbf{LIMODENet-AE} & \textbf{35.61}{\tiny$\pm$0.11}
        & \textbf{0.9738}{\tiny$\pm$0.0011} & \textbf{81.72}{\tiny$\pm$0.68} \\
      & U-Net & 34.89{\tiny$\pm$0.27} & 0.9652{\tiny$\pm$0.0030} & 78.43{\tiny$\pm$1.01} \\
      & CNN-AE & 34.08{\tiny$\pm$0.06} & 0.9552{\tiny$\pm$0.0008} & 71.90{\tiny$\pm$0.21} \\
    \midrule
    \multicolumn{5}{l}{\emph{$3$ and $4$\,dB are above the decoding threshold;
      all models tie. See \cref{supp:abovecliff}.}} \\
    \bottomrule
  \end{tabular}
  \caption{\textbf{The ranking also holds when the channel, rather than the
  compression, is varied.} Iso-parameter three-seed comparison at JPEG quality $100$
  and two $E_s/N_0$ settings; $60$ epochs, val-selected epochs, neutral Spiking-CNN
  judge. Absolute values are comparable \emph{within} a row block only.}
  \label{supp:tab:snr}
\end{table}

\section{Judge Control and Label-Free Semantic Validation}
\label{supp:semantic}

Reconstruct-then-classify accuracy requires a downstream classifier, and the
natural choice, our own ANN LIMODENet, shares an architecture family with one
of
the three encoders being compared. That is a confound, so we re-scored the identical reconstructions with two judges unrelated to all three reconstructors: a Spiking-CNN and a SEW-ResNet~\cite{fang2021deep} (Table~\ref{supp:tab:judge}). Against the
CNN-AE the margin is large under every judge ($+4.2$ to $+8.4$\,pp) and the
conclusion is unaffected. Against the U-Net, however, the margin falls from
$+4.50$\,pp under the LIMODENet judge to $+1.34$ and $+1.84$\,pp under the neutral ones, and under the Spiking-CNN judge the error bars overlap. A formal
two-one-sided-equivalence test (TOST, $\delta_0{=}1.27$\,pp, calibrated against
independent positive/negative controls from the ablation sweep, same procedure as
\cref{supp:abovecliff}) on the Spiking-CNN row is \emph{inconclusive}: the $90\%$ CI of the $+1.34$\,pp gap runs from $-0.16$ to $+2.84$\,pp, straddling
$\pm\delta_0$ in both directions. We therefore scope the claim accordingly:
\textbf{at $1$\,dB, LIMODENet's advantage over the U-Net is a fidelity result
rather than a downstream-accuracy result.}

\begin{table}[h]
  \centering
  \footnotesize
  \setlength{\tabcolsep}{3pt}
  \begin{tabular}{lccc}
    \toprule
    & \multicolumn{3}{c}{recon$\to$clf Top-1 (\%), $3$ seeds} \\
    \cmidrule(lr){2-4}
    Judge & LIMODENet-AE & U-Net & CNN-AE \\
    \midrule
    LIMODENet-ANN$^{\dagger}$ & 76.42{\tiny$\pm$1.35}
      & 71.92{\tiny$\pm$1.91} & 67.98{\tiny$\pm$0.69} \\
    Spiking-CNN & 78.50{\tiny$\pm$0.83}
      & 77.17{\tiny$\pm$0.89} & 71.89{\tiny$\pm$0.77} \\
    SEW-ResNet & 81.55{\tiny$\pm$0.60}
      & 79.71{\tiny$\pm$0.58} & 77.35{\tiny$\pm$0.74} \\
    \midrule
    \multicolumn{4}{l}{\emph{LIMODENet margin (LIMODENet-ANN / Sp.-CNN / SEW)}} \\
    \quad vs.\ U-Net  & \multicolumn{3}{c}{$+4.50$ \,/\, $+1.34^{\circ}$ \,/\, $+1.84$} \\
    \quad vs.\ CNN-AE & \multicolumn{3}{c}{$+8.44$ \,/\, $+6.61$ \,/\, $+4.20$} \\
    \bottomrule
  \end{tabular}
  \caption{\textbf{A negative control: the semantic margin over the U-Net is an
  artifact of the judge.} $^{\dagger}$Confounded judge; the other two are
  architecturally unrelated to all three reconstructors. $^{\circ}$Under the
  Spiking-CNN judge the LIMODENet--U-Net difference lies within overlapping
  error bars.}
  \label{supp:tab:judge}
\end{table}

By the data-processing inequality~\cite{cover2006elements}, for the Markov chain
$S \to \widehat{X} \to Z_{\mathrm{judge}}$ that any reconstruct-then-classify
pipeline forms, $I(S;Z_{\mathrm{judge}}) \le I(S;\widehat{X})$: a downstream
classifier can only report on the fraction of the restored image's information
that \emph{its own} decision boundary happens to expose. A judge-\emph{free}
comparison must therefore query the representation \emph{before} any
classification head commits to a decision boundary. We embed clean, degraded,
and each reconstructor's output with a frozen encoder $\phi$, the LIMODENet backbone itself, pretrained on Tiny-ImageNet and never fine-tuned on EuroSAT, and compare embeddings with three label-free metrics: paired cosine
similarity ($\mathrm{PSS}$), normalized latent distortion ($D_{\mathrm{sem}}$),
and linear centered kernel alignment
($\mathrm{CKA}$~\cite{kornblith2019similarity}), all computed on the pooled
$128$-d feature preceding $\phi$'s own classification head:
\begin{align*}
  \mathrm{PSS} &= \tfrac{1}{N}\textstyle\sum_i
    \cos\!\big(\phi(x_i),\,\phi(\widehat{x}_i)\big),\\
  D_{\mathrm{sem}}(x_i,\widehat{x}_i) &=
    \frac{\lVert \phi(x_i)-\phi(\widehat{x}_i)\rVert_2}
         {\lVert \phi(x_i)\rVert_2+\varepsilon},\\
  \mathrm{CKA}(Z,\widehat{Z}) &=
    \frac{\langle K_c,\widehat{K}_c\rangle_F}
         {\lVert K_c\rVert_F\lVert \widehat{K}_c\rVert_F},
\end{align*}
with $K_c,\widehat{K}_c$ the centered Gram matrices of the clean and restored
embeddings.

\begin{table}[h]
  \centering
  \footnotesize
  \setlength{\tabcolsep}{4pt}
  \begin{tabular}{llccc}
    \toprule
    & Reconstructor & PSS ($\uparrow$) & $D_{\mathrm{sem}}$ ($\downarrow$) & CKA ($\uparrow$) \\
    \midrule
    \multirow{3}{*}{\rotatebox{90}{$1$dB/$100$q}}
      & \textbf{LIMODENet-AE} & \textbf{0.9967}{\tiny$\pm$0.0002}
        & \textbf{0.0675}{\tiny$\pm$0.0026} & \textbf{0.9931}{\tiny$\pm$0.0004} \\
      & U-Net  & 0.9949{\tiny$\pm$0.0007} & 0.0854{\tiny$\pm$0.0054} & 0.9876{\tiny$\pm$0.0016} \\
      & CNN-AE & 0.9928{\tiny$\pm$0.0004} & 0.0998{\tiny$\pm$0.0026} & 0.9833{\tiny$\pm$0.0009} \\
    \midrule
    \multirow{3}{*}{\rotatebox{90}{$1$dB/$50$q}}
      & \textbf{LIMODENet-AE} & \textbf{0.9892}{\tiny$\pm$0.0002}
        & \textbf{0.1234}{\tiny$\pm$0.0008} & \textbf{0.9707}{\tiny$\pm$0.0006} \\
      & U-Net  & 0.9868{\tiny$\pm$0.0009} & 0.1366{\tiny$\pm$0.0047} & 0.9631{\tiny$\pm$0.0025} \\
      & CNN-AE & 0.9839{\tiny$\pm$0.0001} & 0.1501{\tiny$\pm$0.0004} & 0.9562{\tiny$\pm$0.0001} \\
    \midrule
    \multirow{3}{*}{\rotatebox{90}{$1$dB/$10$q}}
      & \textbf{LIMODENet-AE} & \textbf{0.9621}{\tiny$\pm$0.0005}
        & \textbf{0.2365}{\tiny$\pm$0.0011} & \textbf{0.8859}{\tiny$\pm$0.0025} \\
      & U-Net  & 0.9597{\tiny$\pm$0.0004} & 0.2422{\tiny$\pm$0.0010} & 0.8784{\tiny$\pm$0.0008} \\
      & CNN-AE & 0.9581{\tiny$\pm$0.0001} & 0.2496{\tiny$\pm$0.0005} & 0.8736{\tiny$\pm$0.0013} \\
    \midrule
    \multirow{3}{*}{\rotatebox{90}{$2$dB/$100$q}}
      & \textbf{LIMODENet-AE} & \textbf{0.9976}{\tiny$\pm$0.0001}
        & \textbf{0.0546}{\tiny$\pm$0.0016} & \textbf{0.9901}{\tiny$\pm$0.0005} \\
      & U-Net  & 0.9959{\tiny$\pm$0.0006} & 0.0738{\tiny$\pm$0.0060} & 0.9857{\tiny$\pm$0.0017} \\
      & CNN-AE & 0.9935{\tiny$\pm$0.0002} & 0.0958{\tiny$\pm$0.0019} & 0.9794{\tiny$\pm$0.0008} \\
    \bottomrule
  \end{tabular}
  \caption{\textbf{A judge-free re-run of the downstream comparison separates
  cleanly on all three metrics, at every sub-threshold condition.} Frozen
  Tiny-ImageNet-pretrained LIMODENet backbone, never fine-tuned on EuroSAT; three
  seeds, mean$\pm$std. LIMODENet leads on all three metrics with non-overlapping
  error bars at all four conditions, $24/24$ comparisons in total.}
  \label{supp:tab:semmetrics}
\end{table}

The encoder above is supervised (Tiny-ImageNet classification), not the
self-supervised, domain-general representation the label-free-preservation
literature recommends, so we repeat the $1$\,dB/$100$q condition with DINOv2
ViT-S/$14$ \cite{oquab2023dinov2}, self-supervised and never exposed to any classification label (Table~\ref{supp:tab:dinov2}):

\begin{table}[h]
  \centering
  \footnotesize
  \setlength{\tabcolsep}{5pt}
  \begin{tabular}{lccc}
    \toprule
    Reconstructor & PSS ($\uparrow$) & $D_{\mathrm{sem}}$ ($\downarrow$) & CKA ($\uparrow$) \\
    \midrule
    \textbf{LIMODENet-AE} & \textbf{0.7960}{\tiny$\pm$0.0098}
      & \textbf{0.6108}{\tiny$\pm$0.0157} & \textbf{0.7991}{\tiny$\pm$0.0107} \\
    U-Net  & 0.7459{\tiny$\pm$0.0103} & 0.6848{\tiny$\pm$0.0151} & 0.7570{\tiny$\pm$0.0125} \\
    CNN-AE & 0.6937{\tiny$\pm$0.0143} & 0.7560{\tiny$\pm$0.0180} & 0.6807{\tiny$\pm$0.0113} \\
    \bottomrule
  \end{tabular}
  \caption{\looseness=-1 \textbf{An independent, self-supervised encoder reproduces the identical
  ranking.} DINOv2 ViT-S/$14$; $1$\,dB/$100$q, three seeds, mean$\pm$std. All
  three metrics separate with non-overlapping error bars, same order as
  Table~\ref{supp:tab:semmetrics}. Absolute values are far lower (domain gap
  from natural-image pretraining) but the relative ordering is exactly
  reproduced.}
  \label{supp:tab:dinov2}
\end{table}

\paragraph{A worst-case guarantee, not a measurement, agrees with every metric
above.} We add one closed-form, label-free sufficient condition: a certificate
that a prediction \emph{cannot} have changed. Let $\phi$ be a frozen encoder
and $h_c(z)=w_c^\top z+b_c$ the linear head's $c$-th logit on embedding $z$, so
that $c^*=\arg\max_c h_c(\phi(x))$ is the predicted class for the clean image
$x$. For every competing class $j\neq c^*$, define the clean-image prediction
margin $m_j(x) = h_{c^*}(z) - h_j(z)$, $z=\phi(x)$. Let $\widetilde
z=\phi(\widetilde x)$ be the embedding of a degraded or restored version of
$x$, and $\Delta z=\widetilde z-z$. Cauchy--Schwarz gives
\begin{equation*}
  h_{c^*}(\widetilde z)-h_j(\widetilde z)
  \;\geq\; m_j(x) - \lVert w_{c^*}-w_j\rVert_2\,\lVert\Delta z\rVert_2,
\end{equation*}
which stays positive for every $j\neq c^*$ whenever
\begin{equation*}
  \lVert\Delta z\rVert_2 \;<\; \min_{j\neq c^*}
  \frac{m_j(x)}{\lVert w_{c^*}-w_j\rVert_2},
\end{equation*}
in which case $c^*$ cannot have changed, regardless of what $\widetilde x$
actually is. The Certified Semantic Preservation Rate (CSPR) is the fraction
of the test set satisfying this bound. The certificate is
\emph{model-relative}, and we evaluate it using the Tiny-ImageNet encoder's
own trained linear head, so it certifies the encoder's $200$-way Tiny-ImageNet
decision, not an EuroSAT-relevant one, and should be read as a demonstration
of the mechanism rather than a task-specific claim. Raw channel degradation is
\emph{provably uncertifiable for every single test image} (median
$\lVert\Delta z\rVert_2 \approx\!41$ against a median certified radius of
$\approx\!0.28$). Restoration reduces the perturbation $25$--$28\times$ and
yields a first non-zero certified rate: LIMODENet $9.32\pm0.79\%$, the U-Net
$7.23\pm0.46\%$, the CNN-AE $6.53\pm0.17\%$, non-overlapping at three seeds, the same order as every other metric in this section. This is not evidence
that LIMODENet satisfies the injectivity condition of
Proposition~\ref{supp:p5} (\cref{sec:probing} already shows it does not, by one to two orders of magnitude), only that an honest worst-case guarantee, applied
without cherry-picking the bound, detects the same advantage every looser
metric does.

\paragraph{The semantic-embedding advantage holds, and partly strengthens, at half the parameter budget.}
Re-running PSS/$D_{\mathrm{sem}}$/CKA/CSPR on the existing \texttt{nano} checkpoints (pure inference, no new training, Table~\ref{supp:tab:semmetrics-nano}):
\begin{table}[h]
  \centering
  \footnotesize
  \setlength{\tabcolsep}{2.5pt}
  \begin{tabular}{lcccc}
    \toprule
    Nano & PSS\,$\uparrow$ & $D_{\mathrm{sem}}$\,$\downarrow$ & CKA\,$\uparrow$ & CSPR\,(\%) \\
    \midrule
    \textbf{LIMODENet-AE} & \textbf{0.9955}{\tiny$\pm$0.0001} & \textbf{0.0802}{\tiny$\pm$0.0009} & \textbf{0.9902}{\tiny$\pm$0.0001} & \textbf{7.36}{\tiny$\pm$0.32} \\
    U-Net  & 0.9911{\tiny$\pm$0.0019} & 0.1111{\tiny$\pm$0.0111} & 0.9776{\tiny$\pm$0.0055} & 5.93{\tiny$\pm$0.48} \\
    CNN-AE & 0.9920{\tiny$\pm$0.0004} & 0.1052{\tiny$\pm$0.0030} & 0.9812{\tiny$\pm$0.0011} & 6.27{\tiny$\pm$0.20} \\
    \bottomrule
  \end{tabular}
  \caption{\looseness=-1 \textbf{The judge-free and certified-guarantee advantages both survive
  at half the parameter budget.} Nano band ($0.44$M for the full autoencoder), $1$\,dB/$100$q, three
  seeds, mean$\pm$std; frozen Tiny-ImageNet encoder, same protocol as
  Table~\ref{supp:tab:semmetrics}. The $D_{\mathrm{sem}}$ margin over the U-Net
  \emph{grows} relative to the tiny band ($0.0179\!\to\!0.0309$), echoing the
  fidelity finding (\cref{sec:ablations}) that the U-Net margin strengthens rather than shrinks at half budget. As in Table~\ref{supp:tab:nano}, the two \emph{baselines} swap at this budget: the CNN-AE edges the U-Net on all four metrics here, having trailed it on all three in the tiny band (Table~\ref{supp:tab:semmetrics}). Only LIMODENet's first place is stable across bands.}
  \label{supp:tab:semmetrics-nano}
\end{table}

\section{Convergence Check}
\label{supp:convergence}

A short schedule invites the objection that the baselines were merely cut off
before they caught up. Table~\ref{supp:tab:converge} answers it by running the
identical three-seed comparison at both a $25$- and a $60$-epoch budget. First,
$25$ epochs was \emph{not} convergence for any model, since all three gain $1.5$--$1.9$\,dB, so the shorter budget understated every model rather than
the baselines alone. Second, and decisively, the gaps \emph{grow} once all
three converge: against the CNN-AE the margin moves from $+1.44$ to
$+1.75$\,dB, and against the U-Net from $+0.64$ to $+1.07$\,dB. At $60$ epochs
every validation curve is flat to within $0.003$\,dB/epoch, so this is a
genuine convergence point rather than another arbitrary cutoff.

\begin{table}[t]
  \centering
  \footnotesize
  \setlength{\tabcolsep}{2.5pt}
  \begin{tabular}{lcccc}
    \toprule
    & \multicolumn{2}{c}{PSNR (dB)} & \multicolumn{2}{c}{recon$\to$clf (\%)$^{\ddagger}$} \\
    \cmidrule(lr){2-3}\cmidrule(lr){4-5}
    & $25$\,ep$^{\dagger}$ & $60$\,ep & $25$\,ep$^{\dagger}$ & $60$\,ep \\
    \midrule
    \textbf{LIMODENet-AE} & 35.49{\tiny$\pm$0.18} & \textbf{37.39{\tiny$\pm$0.14}}
                          & 67.08{\tiny$\pm$0.63} & \textbf{76.42{\tiny$\pm$1.35}} \\
    U-Net                 & 34.85{\tiny$\pm$0.33} & 36.32{\tiny$\pm$0.33}
                          & 64.61{\tiny$\pm$1.00} & 71.92{\tiny$\pm$1.91} \\
    CNN-AE                & 34.05{\tiny$\pm$0.17} & 35.64{\tiny$\pm$0.18}
                          & 61.03{\tiny$\pm$0.73} & 67.98{\tiny$\pm$0.69} \\
    \midrule
    \emph{margin vs.\ U-Net}  & $+0.64$ & $\mathbf{+1.07}$
                              & $+2.47$ & $\mathbf{+4.50}$ \\
    \emph{margin vs.\ CNN-AE} & $+1.44$ & $\mathbf{+1.75}$
                              & $+6.05$ & $\mathbf{+8.44}$ \\
    \bottomrule
  \end{tabular}
  \caption{\looseness=-1 \textbf{Training to convergence widens LIMODENet's margin rather than
  closing it.} $^{\dagger}$The $25$-epoch columns predate the validation-split
  protocol and select the epoch on test. $^{\ddagger}$Both recon$\to$clf
  columns use the LIMODENet-ANN judge so the two budgets remain comparable to
  each other; the PSNR columns need no judge at all.}
  \label{supp:tab:converge}
\end{table}

\section{Recent SOTA Lightweight RS Classifiers}
\label{supp:sota-classification}

\Cref{sec:classification} argues that saturated RS scene classification does not
separate architectures; this appendix substantiates that claim by placing
LIMODENet's EuroSAT numbers next to five recent lightweight or efficiency-oriented
classifiers (\cref{supp:tab:sota}). We stress up front that
\textbf{\cref{supp:tab:sota} is not a controlled comparison}: the entries differ in
input resolution ($64$ vs.\ $224$--$256$\,px), pre-training corpus, train/test
split ratio, and augmentation, none of which we can equalise from published
numbers. It is included so that a reader familiar with the RS classification
literature can see where our backbone sits, not as evidence for any claim in
the paper: the discriminating result is the iso-parameter restoration
comparison of
\cref{sec:recon-compare}.

Three observations survive the protocol noise. \textbf{(i)} The only
protocol-matched external point is SceneMixer~\cite{scenemixer2025}: same dataset,
native $64\times64$ input, same depthwise$+$pointwise design family. LIMODENet's
from-scratch $64$\,px \emph{tiny} model ($96.5\%$, 3 seeds) and its legacy
ImageNet-finetuned number ($98.45\%$, single seed) both exceed SceneMixer's $93.90\%$ overall accuracy, but this is one dataset at one operating point and we do not build on it.
\textbf{(ii)} The $224$--$256$\,px methods~\cite{yang2022focalnet,stconvnext2025,%
kcn2024,jeevan2025wavemix} operate on EuroSAT only after $2$--$4\times$ upsampling
from the native $64$\,px, and carry $13$--$29$\,M parameters against LIMODENet's
$0.69$\,M; their reported accuracies ($\sim$96--98\%) sit within the same
saturated band. \textbf{(iii)} Consistent with our own width-ladder finding
(Tiny$\to$Base moves EuroSAT by $+0.7$\,pp for $4\times$ the parameters), the
spread across all six methods is under $\sim$5\,pp despite a $40\times$ range in
parameter count, which is the quantitative form of the ``benchmark is saturated''
claim.

\begin{table*}[h]
  \centering
  \footnotesize
  \setlength{\tabcolsep}{3pt}
  \begin{tabular}{lcccl}
    \toprule
    Method & EuroSAT top-1 & Input & Params & Protocol note \\
    \midrule
    \textbf{LIMODENet} (tiny, scratch) & $96.54${\tiny$\pm$0.17} & $64$ & $0.69$M
      & 3 seeds, val-selected \\
    \textbf{LIMODENet} (IN$\to$FT)$^{\dagger}$ & $98.45$ & $64$ & $0.69$M
      & single seed, best ckpt \\
    \midrule
    SceneMixer~\cite{scenemixer2025} & $93.90$ & $64$ & low
      & native res., same DW+PW family \\
    KCN / KonvNeXt~\cite{kcn2024} & ${\sim}96$ & $224$ & ${\geq}28$M
      & ConvNeXt+KAN head, upsampled \\
    STConvNeXt~\cite{stconvnext2025} & --$^{\ast}$ & $224$ & ${\sim}13$M
      & reports NWPU/AID/UCM, not EuroSAT \\
    FocalNet-T~\cite{yang2022focalnet} & --$^{\ast}$ & $224$ & $28.6$M
      & ImageNet backbone, no EuroSAT eval \\
    WaveMix~\cite{jeevan2025wavemix} & top of survey & $256$ & varies
      & $64{\to}256$ upsample, backbone study \\
    \bottomrule
  \end{tabular}
  \caption{\looseness=-1 \textbf{LIMODENet against recent lightweight / efficient RS
  classifiers on EuroSAT.} Not a controlled comparison: entries differ in input
  resolution, pre-training, split, and augmentation. $^{\dagger}$Legacy
  single-seed protocol, not comparable to the 3-seed rows.
  $^{\ast}$No EuroSAT number in the source; the method targets other RS
  benchmarks or ImageNet. Included for positioning only; the paper's claims rest
  on \cref{sec:recon-compare}.}
  \label{supp:tab:sota}
\end{table*}

\section{Scaling: Pretraining and Capacity}
\label{supp:scaling}

Because ImageNet-1k~\cite{deng2009imagenet} pre-training is not reproducible on our hardware, we pre-train the family on \textbf{Tiny-ImageNet} (a $200$-class, $100$k-image subset of ImageNet) at the \emph{same} $64{\times}64$ resolution used for the RS benchmarks
(Table~\ref{supp:tab:tin}). Scaling nano$\to$base on this corpus, a
$6.8\times$ parameter increase in which width is the \emph{only}
hyper-parameter that moves, gains $+5.0$\,pp on Tiny-ImageNet ($46.6\!\to\!51.6\%$), whereas the same nano$\to$base step gains only $+1.6$\,pp on EuroSAT. Capacity returns track distance from saturation rather than dataset
size.

\begin{table}[t]
  \centering
  \small
  \setlength{\tabcolsep}{6pt}
  \begin{tabular}{lccc}
    \toprule
    Scale & Params & Tiny-ImageNet top-1 (\%) & GPU-h \\
    \midrule
    nano  & 0.41M & 46.64 & 0.9 \\
    tiny  & 0.72M & 47.29 & 1.1 \\
    small & 1.58M & 50.85 & 1.3 \\
    base  & 2.78M & \textbf{51.61} & 1.7 \\
    \midrule
    big$^\ast$ & 5.85M & 51.06 & 3.5 \\
    \bottomrule
  \end{tabular}
  \caption{\textbf{Tiny-ImageNet pre-training}, $60$ epochs at $64{\times}64$,
  single seed. These checkpoints initialize the IN$\to$FT arm of the scaling study. Parameter counts exceed those of Table~\ref{tab:family} because the head here is $200$-way rather than $10$-way. $^\ast$\,\emph{big} is off the width ladder: it changes depth, MLP
  ratio and kernel as well as width, and despite $2.1\times$ the parameters of
  \emph{base} it does not exceed it.}
  \label{supp:tab:tin}
\end{table}

If capacity returns track headroom, pretraining, which also closes headroom by
giving the network useful features before it ever sees an RS label, should
show the same signature, and should erode as capacity itself
closes that gap. We fine-tune every TinyIN-pretrained scale on every RS benchmark, $3$ seeds each ($75$ runs), and compare against the from-scratch numbers of Table~\ref{tab:rs-datasets} cell by cell (Table~\ref{supp:tab:tinftdelta}):

\begin{table}[h]
  \centering
  \footnotesize
  \setlength{\tabcolsep}{4pt}
  \begin{tabular}{lccccc}
    \toprule
    $\Delta_{\mathrm{TinyIN}}$ & EuroSAT & NWPU & PatternNet & RSICB & UCMerced \\
    \midrule
    nano & $+1.9$ & $+11.4$ & $+3.5$ & $+2.2$ & $\mathbf{+36.3}$ \\
    tiny & $+1.2$ & $+8.1$ & $+2.7$ & $+1.5$ & $+31.9$ \\
    small & $+0.5$ & $+6.7$ & $+1.7$ & $+1.2$ & $+25.9$ \\
    base & $+0.5$ & $+5.7$ & $+1.3$ & $+1.0$ & $+21.8$ \\
    big & $+0.2$ & $+3.1$ & $+1.0$ & $+0.5$ & $+15.5$ \\
    \bottomrule
  \end{tabular}
  \caption{\looseness=-1 \textbf{The pretraining benefit decays monotonically with capacity on
  every one of the five RS benchmarks.} $\Delta_{\mathrm{TinyIN}}$ is TinyIN$\to$FT
  minus from-scratch, same cell, both $3$-seed means. Every one of the $20$ scale transitions is non-increasing, with a single tie (EuroSAT, small$\to$base).}
  \label{supp:tab:tinftdelta}
\end{table}

The largest gains land exactly where headroom is largest: $+36.3$\,pp on
UCMerced and $+11.4$\,pp on NWPU at nano, against ${\leq}3.5$\,pp on the three
near-saturated benchmarks at the same scale. The substitution is strong enough
that a \emph{nano} model with TinyIN pretraining beats a \emph{big} model
trained from scratch on the hardest benchmark by a wide margin ($92.54$ vs.\
$79.31\%$ on UCMerced, $+13.2$\,pp with $14.9\times$ fewer parameters), and edges it on NWPU ($86.24$ vs.\ $85.01\%$).

The family design lets us separate width from off-ladder capacity. Along the
width ladder accuracy rises monotonically with $C$ on every one of the six
corpora checked, including the single-seed Tiny-ImageNet pretraining check
($46.64\!\to\!47.29\!\to\!50.85\!\to\!51.61$). The off-ladder \emph{big}
configuration keeps $C{=}256$ but adds depth, MLP ratio, and kernel, for
$2.1\times$ the parameters of \emph{base}. On Tiny-ImageNet pretraining it does
not improve on \emph{base} ($51.06$ vs.\ $51.61$, single-seed); we confirmed this with three seeds on two RS benchmarks chosen to span the headroom axis (Table~\ref{supp:tab:basebig}):

\begin{table}[h]
  \centering
  \footnotesize
  \setlength{\tabcolsep}{4pt}
  \begin{tabular}{lccc}
    \toprule
    Benchmark & Base & Big & disjoint? \\
    \midrule
    EuroSAT (saturated) & 97.26{\tiny$\pm$0.17} & 97.54{\tiny$\pm$0.14} & no \\
    NWPU (headroom)     & 82.94{\tiny$\pm$0.40} & 85.01{\tiny$\pm$0.33} & \textbf{yes, $+2.07$} \\
    \bottomrule
  \end{tabular}
  \caption{\looseness=-1 \textbf{Whether the off-ladder \emph{big} configuration beats
  \emph{base} depends on headroom, not on a fixed capacity budget.} $3$ seeds, mean$\pm$std, epochs selected on a held-out split.}
  \label{supp:tab:basebig}
\end{table}

On the saturated benchmark \emph{big} does not separate from \emph{base},
matching the single-seed pretraining check; on the benchmark with real
headroom it does, by a disjoint $+2.07$\,pp. Capacity beyond the ladder pays
exactly where the ladder itself still has headroom to give, and not
otherwise.

\section{Recovery Pipeline: Off-Operating-Point}
\label{supp:offpoint}

\begin{table}[t]
  \centering
  \footnotesize
  \setlength{\tabcolsep}{3pt}
  \begin{tabular}{llccc}
    \toprule
    System & Metric & $10$q & $50$q & $100$q \\
    \midrule
    \multirow{3}{*}{\shortstack[l]{Recon.\\$\to$\,clf.}}
      & PSNR (dB)  & 30.63 & 33.98 & 38.25 \\
      & SSIM       & 0.8585 & 0.9340 & 0.9681 \\
      & Top-1 (\%) & 47.44 & 66.87 & 85.96 \\
    \midrule
    \multirow{3}{*}{\shortstack[l]{Degraded-\\trained clf.}}
      & PSNR (dB)  & \multicolumn{3}{c}{\emph{n/a --- emits no image}} \\
      & SSIM       & \multicolumn{3}{c}{\emph{n/a --- emits no image}} \\
      & Top-1 (\%) & \textbf{81.81$\pm$0.22} & \textbf{87.89$\pm$0.11}
                   & \textbf{88.31$\pm$0.23} \\
    \midrule
    \multicolumn{2}{l}{Top-1 margin} & $+34.4$ & $+21.0$ & $+2.4$ \\
    \bottomrule
  \end{tabular}
  \caption{\textbf{The two systems are not comparable on the same axes, and
  that is the point.} $3$ seeds, mean$\pm$std. The classifier has no PSNR or SSIM to report, because it discards the scene, whereas restoration returns the
  image. This bounds the label-accuracy claim rather than ranking the systems: the bound is local, and reverses off the training operating point. The reconstruct-then-classify row is the original pipeline reconstructor; the iso-parameter model of Table~\ref{tab:modern} scores $76.42\%$ at $100$q under the LIMODENet-ANN judge (Table~\ref{tab:offpoint}) and loses to the classifier by more.}
  \label{tab:killtest}
\end{table}

\begin{figure*}[t]
  \centering
  \includegraphics[width=\textwidth]{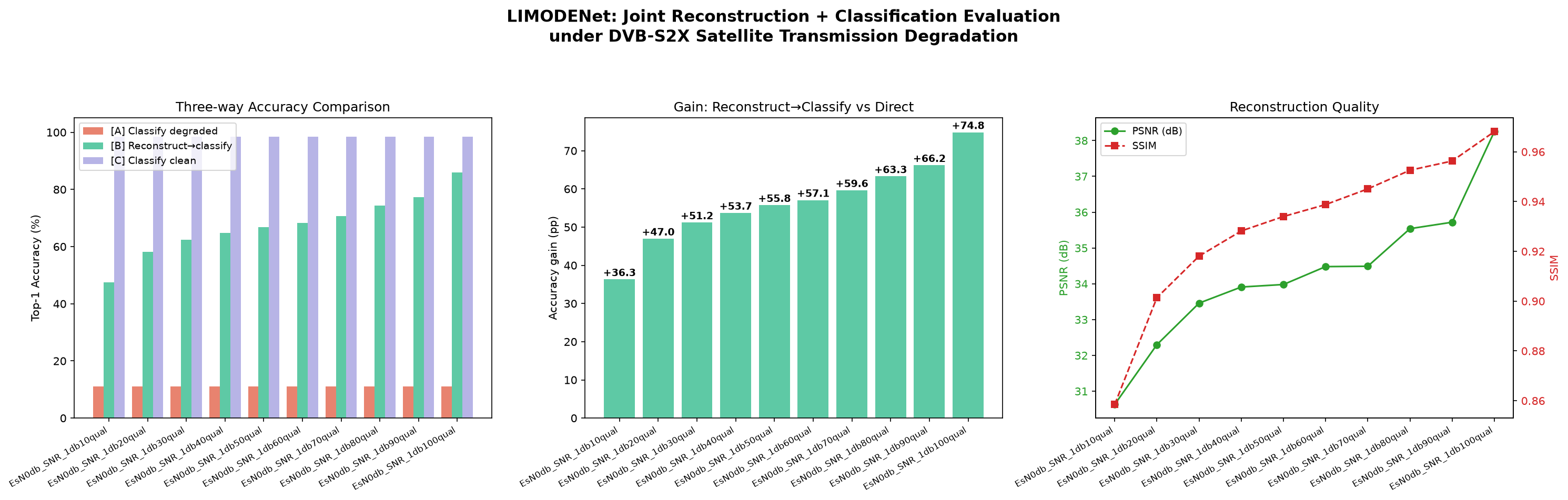}
  \caption{\textbf{Reconstruction recovers most of the label information a
  DVB-S2X channel destroys (EuroSAT).} The left panel plots three-way top-1
  accuracy against JPEG quality at $1$\,dB for (A)~classifying the degraded
  image (a clean-trained classifier applied zero-shot, $11.1\%$), (B)~reconstruct-then-classify, and (C)~the clean
  oracle. \emph{Note:} pipeline~(A) is a clean-trained classifier applied
  zero-shot and is a weak baseline; the gain over it is \textbf{not} our claim (see Table~\ref{tab:killtest}).}
  \label{fig:recovery}
\end{figure*}

\begin{figure}[t]
  \centering
  \includegraphics[width=\linewidth]{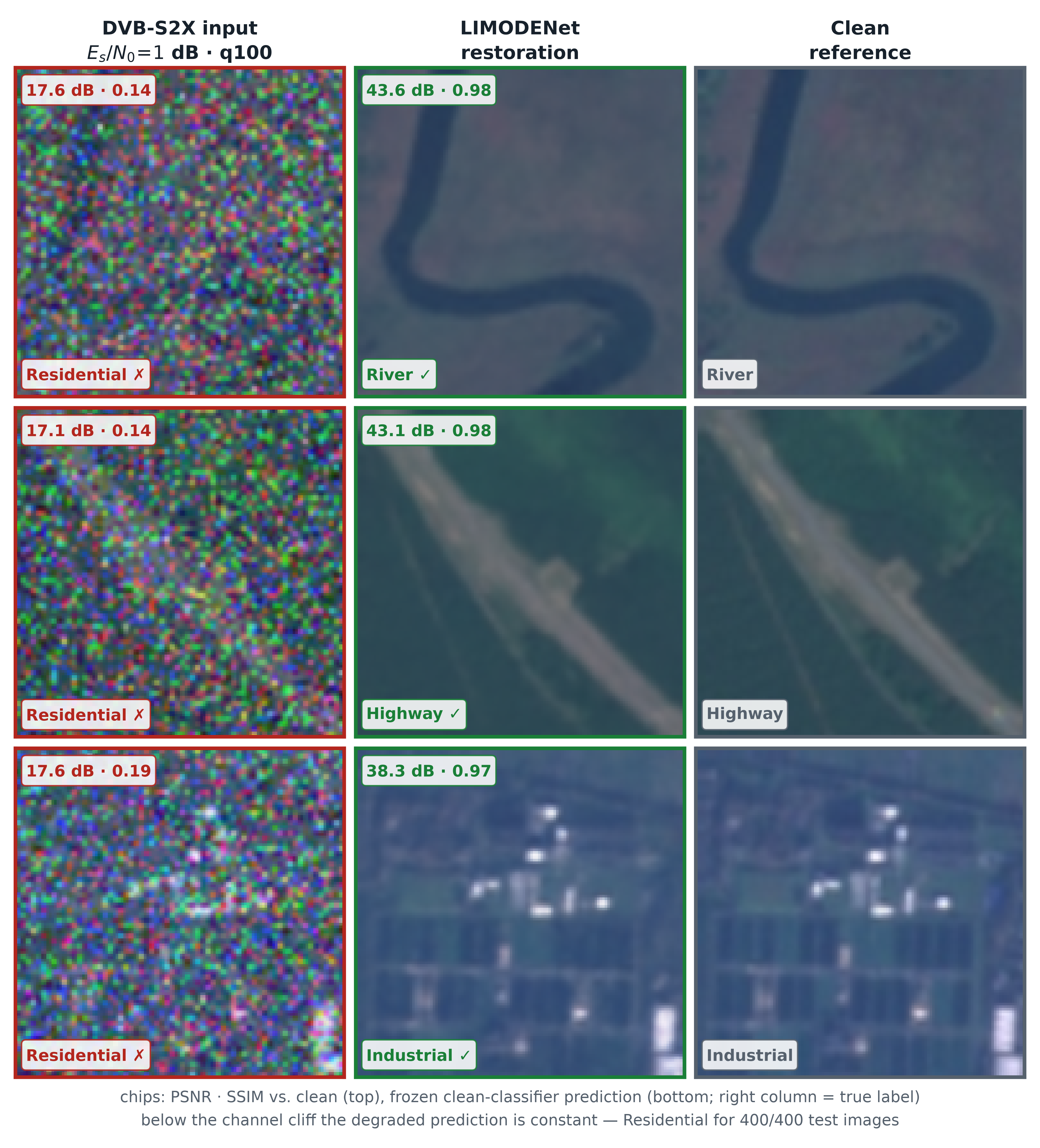}
  \caption{\textbf{The backbone recovers both the image and its label}
  (classes \emph{River}, \emph{Highway}, \emph{Industrial}; $1$\,dB, $100$q).
  The DVB-S2X-degraded input (left) is noise, and the frozen clean classifier
  does not merely err on it but \emph{collapses to a constant}: it predicts
  \emph{Residential} for $400/400$ test images spanning all ten classes, and
  likewise at $1$\,dB/q10 and $2$\,dB (it recovers only at $3$\,dB, $95.0\%$).
  The backbone's reconstruction (middle) restores both the image and the
  correct label against the clean reference (right). Because pipeline~(A) is
  a constant function below the channel cliff, the label flip is reported as
  a qualitative illustration only, not as evidence of semantic gain
  (cf.\ Table~\ref{tab:killtest}).}
  \label{fig:qual}
\end{figure}

Figs.~\ref{fig:recovery} and~\ref{fig:qual} show the original pipeline's recovery curve and a qualitative example; both are illustrations of the pathway, not evidence for it. Subjected to a channel-mismatch test (trained at $1$\,dB, applied unchanged at $2/3/4$\,dB), the degraded-trained classifier does not merely degrade, it
collapses: from $88.31{\pm}0.23\%$ at its own operating point to
$78.52{\pm}0.48\%$ at $2$\,dB and $47.9/47.8\%$ at $3/4$\,dB, a loss of over
$40$\,pp across three seeds. The restorer degrades too, but by $5.0$\,pp
rather than $40$, an eightfold difference in sensitivity. The kill-test
verdict of Table~\ref{tab:killtest} is therefore \emph{local}: the
degraded-trained classifier's advantage exists only at the operating point it
was trained for, the two systems cross over between $2$ and $3$\,dB, and by
$3$\,dB the ordering has reversed in the restorer's favour by $+23.7$\,pp.

\begin{table}[t]
  \centering
  \footnotesize
  \setlength{\tabcolsep}{4pt}
  \begin{tabular}{lcccc}
    \toprule
    Evaluated at ($100$q) & $1$\,dB$^{\ast}$ & $2$\,dB & $3$\,dB & $4$\,dB \\
    \midrule
    Degraded-trained classifier & \textbf{88.31} & 78.52 & 47.88 & 47.77 \\
    \emph{\quad($\pm$ over $3$ seeds)} & {\tiny$\pm$0.23} & {\tiny$\pm$0.48}
      & {\tiny$\pm$2.05} & {\tiny$\pm$2.05} \\
    \quad\emph{change vs.\ own point} & --- & $-9.8$ & $-40.4$ & $-40.5$ \\
    \midrule
    Frozen restorer $\to$ classify & 76.42 & 74.29 & \textbf{71.54} & \textbf{71.41} \\
    \emph{\quad($\pm$ over $3$ seeds)} & {\tiny$\pm$1.35} & {\tiny$\pm$2.34}
      & {\tiny$\pm$1.01} & {\tiny$\pm$1.01} \\
    \quad\emph{change vs.\ own point} & --- & $-2.1$ & $-4.9$ & $-5.0$ \\
    \bottomrule
  \end{tabular}
  \caption{\textbf{The classifier's advantage is local; the restorer's is not.}
  Both systems are trained at $1$\,dB/$100$q and applied \emph{unchanged} at higher
  $E_s/N_0$. $^{\ast}$their common training point, where the classifier wins by
  $+11.9$\,pp. The crossover lies between $2$ and $3$\,dB.}
  \label{tab:offpoint}
\end{table}

At the milder $3$\,dB channel the degraded input is already informative
($37.5\%$ at $10$q up to $97.85\%$ at $100$q), and reconstruction helps
\emph{only in the harsh, low-quality regime}: $+10.9$\,pp at $10$q and
$+9.9$\,pp at $20$q, crossing over near $40$--$50$q and slightly
\emph{hurting} at high quality ($-1.7$\,pp at $100$q). The reconstruct-then-classify
advantage is therefore regime-specific: large when the channel destroys the
signal, absent once it does not.

\section{Neuromorphic Deployment of the Classifier}
\label{supp:neuromorphic}

\begin{table}[t]
  \centering
  \small
  \setlength{\tabcolsep}{3pt}
  \begin{tabular}{lccc}
    \toprule
    Spiking model & Top-1 (\%) & Energy ($\mu$J) & AC-share \\
    \midrule
    Spiking-CNN           & 92.57{\tiny$\pm$0.56} & \textbf{43.9}{\tiny$\pm$0.1}  & 64.1\% \\
    SEW-ResNet            & 92.59{\tiny$\pm$0.42} & 152.5{\tiny$\pm$2.7} & \textbf{92.9\%} \\
    LIMODENet (scratch)   & 91.86{\tiny$\pm$0.28} & 482.7{\tiny$\pm$12.1} & 62.8\% \\
    LIMODENet (ANN init)  & \textbf{93.55}{\tiny$\pm$0.54} & 477.7{\tiny$\pm$3.1} & 61.9\% \\
    Spiking-MLP           & 50.62{\tiny$\pm$0.71} & 25.3{\tiny$\pm$0.0}  & 0.3\% \\
    \bottomrule
  \end{tabular}
  \caption{\looseness=-1 \textbf{A plain Spiking-CNN Pareto-dominates LIMODENet on clean
  EuroSAT classification.} All models hold ${\sim}0.69$M parameters, run at
  $T{=}8$ from scratch, $3$ seeds, mean$\pm$std, scored with a $45$\,nm SynOps
  proxy. LIMODENet (ANN init) and Spiking-CNN are statistically
  indistinguishable on accuracy while Spiking-CNN spends ${\sim}11\times$ less
  energy. Where the architecture pays off is reconstruction (Table~\ref{tab:modern}).}
  \label{supp:tab:dualhw}
\end{table}

\begin{figure}[t]
  \centering
  \includegraphics[width=\linewidth]{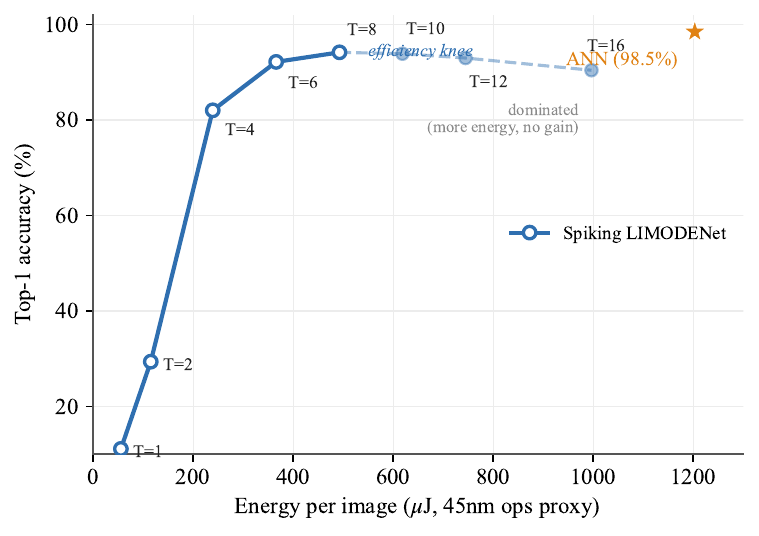}
  \caption{\textbf{Spiking LIMODENet traces an accuracy--energy Pareto front on
  EuroSAT} (representative run). The knee is at $T{=}6$--$8$ ($93.6\%$ 3-seed mean at
  $2.44\times$ lower energy than the ANN, $\star$); $T{\ge}10$ is dominated. Energy is a $45$\,nm SynOps proxy that credits every spike-driven operation; on matched silicon with the measured firing rate, the spiking reconstructor is $3.34\times$ \emph{more} expensive than its ANN (\cref{sec:recon-compare}), so the $2.44\times$ here is an optimistic bound for the classifier, not a measured saving.}
  \label{fig:snn-pareto}
\end{figure}

Table~\ref{supp:tab:snn} and Fig.~\ref{fig:snn-pareto} give the full $T$-sweep behind the headline $93.55{\pm}0.54\%$ at $T{=}8$ number in the main text, and Fig.~\ref{fig:per-layer} shows where the spiking energy goes.

\begin{table}[t]
  \centering
  \small
  \setlength{\tabcolsep}{5pt}
  \begin{tabular}{lcccc}
    \toprule
    Model & $T$ & Top-1 (\%) & Energy (mJ) & Eff.\ vs ANN \\
    \midrule
    ANN  & --  & 98.47 & 1.203 & $1.0\times$ \\
    \midrule
    SNN  & 4   & 82.03 & 0.239 & $5.03\times$ \\
    SNN  & 6   & 92.18 & 0.366 & $3.29\times$ \\
    SNN  & 8   & \textbf{94.18} & 0.493 & $2.44\times$ \\
    SNN  & 10  & 93.89 & 0.619 & $1.95\times$ \\
    SNN  & 12  & 92.99 & 0.745 & $1.62\times$ \\
    SNN  & 16  & 90.43 & 0.996 & $1.21\times$ \\
    \bottomrule
  \end{tabular}
  \caption{Accuracy, SynOps energy, and efficiency of spiking LIMODENet on
  EuroSAT relative to the ANN across time steps $T$, for a single representative run that is not one of the three reported seeds (its $T{=}8$ value, $94.18\%$, sits above the best seed's $94.07\%$). At $T{=}8$ the \textbf{3-seed} mean is $93.55{\pm}0.54\%$. The ANN row is the fine-tuned checkpoint re-evaluated in the spiking pipeline; the classification log reports $98.45\%$ for the same run.}
  \label{supp:tab:snn}
\end{table}

\looseness=-1 Fixing the reconstructor to the ANN LIMODENet and swapping only the downstream
spiking classifier, cross-family SEW-ResNet and Spiking-CNN classify
LIMODENet's reconstructions \emph{better} than a LIMODENet classifier at every
quality, mostly with non-overlapping error bars (Table~\ref{supp:tab:clfswap}).
Read together with Table~\ref{supp:tab:dualhw}, LIMODENet is outperformed as a classifier both on clean images \emph{and} on its own reconstructions, so its justification cannot rest on classification.

\begin{table}[t]
  \centering
  \footnotesize
  \setlength{\tabcolsep}{4pt}
  \begin{tabular}{lccc}
    \toprule
    Spiking classifier & q10 & q50 & q100 \\
    \midrule
    SEW-ResNet              & \textbf{55.30}{\tiny$\pm$3.84} & \textbf{71.43}{\tiny$\pm$3.68} & \textbf{82.10}{\tiny$\pm$1.17} \\
    Spiking-CNN             & 48.63{\tiny$\pm$1.07} & 66.94{\tiny$\pm$2.23} & 79.44{\tiny$\pm$1.47} \\
    LIMODENet (scratch)     & 39.44{\tiny$\pm$1.95} & 59.97{\tiny$\pm$1.40} & 74.80{\tiny$\pm$0.58} \\
    LIMODENet ($+$ANN init) & 39.84{\tiny$\pm$1.35} & 59.37{\tiny$\pm$1.26} & 75.68{\tiny$\pm$2.37} \\
    Spiking-MLP             & 50.39{\tiny$\pm$0.69} & 50.12{\tiny$\pm$1.65} & 50.96{\tiny$\pm$0.70} \\
    \bottomrule
  \end{tabular}
  \caption{\textbf{Swapping the classifier on a fixed LIMODENet reconstructor
  improves accuracy.} Reconstruct-then-classify top-1 accuracy (\%) at
  $1$\,dB DVB-S2X across JPEG quality, $3$ seeds, mean$\pm$std. Cross-family
  SEW-ResNet and Spiking-CNN classify LIMODENet's reconstructions better than
  either LIMODENet classifier does at every quality, non-overlapping in $11$ of
  $12$ cross-family-vs-LIMODENet comparisons.}
  \label{supp:tab:clfswap}
\end{table}

\begin{figure}[t]
  \centering
  \includegraphics[width=\linewidth]{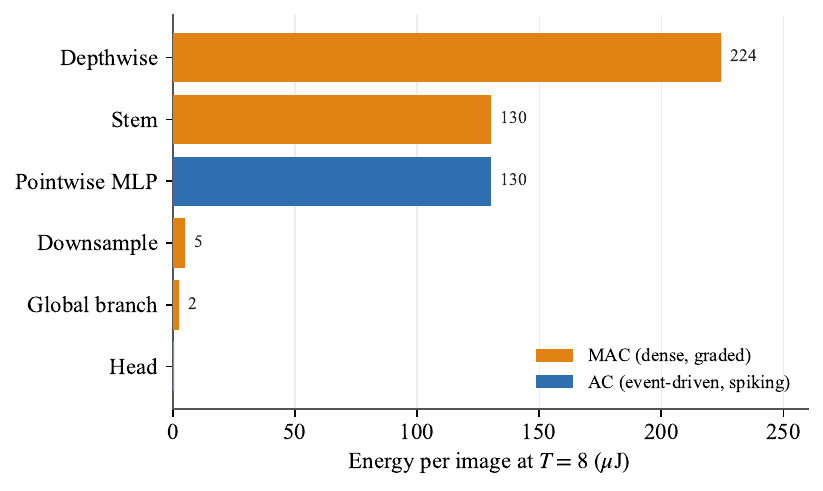}
  \caption{\textbf{Most spiking energy goes to the graded-residual path
  ($T{=}8$).} The depthwise convolutions and the stem consume $67\%$ of
  LIMODENet's spiking-inference energy, because they see a \emph{graded}
  residual stream and therefore run dense MACs; only the pointwise MLP is
  event-driven (AC).}
  \label{fig:per-layer}
\end{figure}

The closest prior work is the ESA study of Kucik and
Meoni~\cite{kucik2021investigating}, which converts a VGG-16 to a spiking network on EuroSAT and estimates energy with the \emph{same} KerasSpiking \texttt{ModelEnergy} tool we adopt. Table~\ref{supp:tab:km} places both under this common tool. At
$20\times$ fewer parameters, LIMODENet improves ANN accuracy by $+3.4$\,pp and
SNN accuracy by $+7$--$8$\,pp; at essentially the same Loihi energy as their
VGG-SNN it is $+7.1$\,pp more accurate.

\begin{table}[t]
  \centering
  \small
  \setlength{\tabcolsep}{4pt}
  \begin{tabular}{lcccc}
    \toprule
    Method & Params & ANN & SNN & Loihi (mJ) \\
    \midrule
    VGG-16~\cite{kucik2021investigating} & ${>}14$M & 95.07 & 85.11 & 4.44 \\
    \textbf{LIMODENet} ($T{=}6$) & \textbf{0.69M} & \textbf{98.47} & 92.18 & 4.70 \\
    \textbf{LIMODENet} ($T{=}8$) & \textbf{0.69M} & \textbf{98.47} & \textbf{93.55} & 6.27 \\
    \bottomrule
  \end{tabular}
  \caption{Apples-to-apples comparison with the ESA on-board SNN baseline
  (energy from the \emph{same} KerasSpiking ModelEnergy tool on Loihi). ANN-on-GPU
  energy is comparable ($79$ vs $70$\,mJ). LIMODENet's SNN top-1 is a $3$-seed mean in the $T{=}8$ row and the single representative run of Table~\ref{supp:tab:snn} in the $T{=}6$ row; test splits differ between the two studies.}
  \label{supp:tab:km}
\end{table}

\section{Full Ablation Grid}
\label{supp:ablation-full}

\cref{sec:ablations} summarizes the ablation findings; this section gives the
full protocol and tables. The comparisons above are whole-model against
whole-model, so they establish that \emph{this} encoder restores better
without showing \emph{which} of its choices is responsible. We therefore ablate the design on the restoration task itself (Table~\ref{supp:tab:ablation}), at $1$\,dB/$q100$, under the protocol used for Table~\ref{tab:modern}: $60$ epochs, val-selected, three
seeds, neutral judge. Every cell varies one axis and stays inside the
\texttt{tiny} iso-parameter band, within $0.03\%$ of the $770{,}051$-parameter
reference.

\begin{table}[t]
  \centering
  \footnotesize
  \setlength{\tabcolsep}{2.5pt}
  \begin{tabular}{lccccc}
    \toprule
    Encoder variant & Params & GMac & PSNR & $\Delta$ & Top-1$^{\ddagger}$ \\
    \midrule
    \textbf{LIMODENet (ours)} & 0.77M & 1.21 & \textbf{37.39}{\tiny$\pm$0.14}
      & --- & \textbf{78.50}{\tiny$\pm$0.83} \\
    \midrule
    \multicolumn{6}{l}{\emph{ODE-specific choices}} \\
    ~~$\alpha{=}1$ (all stages)  & 0.77M & 1.21 & 37.23{\tiny$\pm$0.11} & $-0.16$ & 77.92{\tiny$\pm$0.32} \\
    ~~$\alpha$ learned per block & 0.77M & 1.21 & 37.43{\tiny$\pm$0.15} & $+0.04$ & 78.54{\tiny$\pm$1.08} \\
    ~~Euler at Stage 3          & 0.77M & 1.11 & 37.30{\tiny$\pm$0.21} & $-0.09$ & 78.40{\tiny$\pm$0.90} \\
    \midrule
    \multicolumn{6}{l}{\emph{Architectural choices}} \\
    ~~no global branch          & 0.77M & 1.28 & 36.97{\tiny$\pm$0.13} & $\mathbf{-0.42}$ & 77.28{\tiny$\pm$0.52} \\
    \midrule
    \multicolumn{6}{l}{\emph{Enforcing the injectivity condition}} \\
    ~~spectral-norm, $\alpha L_g{<}1$ & 0.77M & 1.21 & 34.17{\tiny$\pm$0.13} & $\mathbf{-3.22}$ & \textbf{73.54}{\tiny$\pm$1.02} \\
    \bottomrule
  \end{tabular}
  \caption{\textbf{Ablating the design on the restoration task shows that the
  advantage is architectural rather than a consequence of the ODE
  discretization.} Each variant changes one axis, within $0.03\%$ of the
  reference parameter count; protocol matches Table~\ref{tab:modern}. The three ODE cells overlap the reference and are negative results; only the global branch separates, on both metrics. Its replacement is iso-parameter but costs more compute ($1.28$ vs.\ $1.21$\,GMac), so the loss is not bought with a smaller budget; conversely the Euler cell is cheaper ($1.11$), so its null is not bought with extra compute. $^{\ddagger}$Top-1 is
  reconstruct-then-classify accuracy under the neutral Spiking-CNN judge; only
  the spectral-normalized row separates from the reference on this metric.}
  \label{supp:tab:ablation}
\end{table}

Removing the FocalBlock global branch costs $0.42$\,dB with non-overlapping error bars, and it is the only cell besides the spectral-norm variant that separates on
fidelity at all. We tested whether this reflects a general compensation
mechanism for the absence of skip connections (Table~\ref{supp:tab:mechanism}).
The deletion also separates at $1$\,dB/$q10$ ($-0.13$\,dB, disjoint) but not at
$2$\,dB/$q100$ ($-0.08$\,dB, overlapping), so the branch matters at the
harsher of the two channels irrespective of source quality. More decisively,
grafting the identical operator onto the CNN-AE encoder, itself single-path and bottlenecked and therefore exactly the case the compensation reading predicts should benefit, changes nothing ($+0.02$\,dB). A component that is necessary inside one architecture and
inert inside another is not a transferable mechanism.

\begin{table}[t]
  \centering
  \footnotesize
  \setlength{\tabcolsep}{3pt}
  \begin{tabular}{llcccc}
    \toprule
    Test & Condition & variant & reference & $\Delta$ & sep.? \\
    \midrule
    \multicolumn{6}{l}{\emph{Is the global branch necessary?} (remove it from LIMODENet)} \\
    ~~necessity & $1$\,dB/$q100$ & 36.97$\pm$0.13 & 37.39$\pm$0.14 & $-0.42$ & yes \\
    ~~necessity & $1$\,dB/$q10$  & 30.43$\pm$0.06 & 30.56$\pm$0.03 & $-0.13$ & yes \\
    ~~necessity & $2$\,dB/$q100$ & 35.53$\pm$0.20 & 35.61$\pm$0.11 & $-0.08$ & no \\
    \midrule
    \multicolumn{6}{l}{\emph{Is it sufficient?} (add it to the CNN-AE encoder)} \\
    ~~sufficiency & $1$\,dB/$q100$ & 35.66$\pm$0.18 & 35.64$\pm$0.18 & $+0.02$ & no \\
    \bottomrule
  \end{tabular}
  \caption{\textbf{The global branch is necessary in the harshest condition but
  neither condition-general nor transferable.} PSNR in dB, $3$ seeds per cell.
  The sufficiency null is the informative one.}
  \label{supp:tab:mechanism}
\end{table}

Table~\ref{supp:tab:nano} repeats the three-way encoder comparison at half the
\texttt{tiny}-band budget (\texttt{nano}, $0.44$M for the full autoencoder) to test whether the
advantage is an artifact of LIMODENet's $2$--$3\%$ parameter excess. LIMODENet stays first and both margins remain non-overlapping, although the two baselines exchange places on fidelity at this budget. The margin over
the U-Net does not merely survive the reduction, it grows, from $+1.07$\,dB
to $+1.70$\,dB, while the margin over the CNN-AE narrows from $+1.75$ to
$+1.31$\,dB. Skip connections are not free: the decoder must process
concatenated features, so a fixed parameter budget buys less capacity per
path. When parameters are scarce, how information is routed matters more,
not less.

\begin{table}[t]
  \centering
  \footnotesize
  \setlength{\tabcolsep}{2.5pt}
  \begin{tabular}{lcccc}
    \toprule
    Reconstructor & Params & PSNR (dB) & SSIM & Top-1$^{\ddagger}$ \\
    \midrule
    \textbf{LIMODENet-AE} & 0.44M & \textbf{36.65}{\tiny$\pm$0.08}
      & \textbf{0.9619}{\tiny$\pm$0.0003} & \textbf{75.96}{\tiny$\pm$0.31} \\
    CNN-AE                & 0.43M & 35.34{\tiny$\pm$0.12} & 0.9506{\tiny$\pm$0.0016} & 70.85{\tiny$\pm$0.89} \\
    U-Net (skips)         & 0.43M & 34.96{\tiny$\pm$0.48} & 0.9464{\tiny$\pm$0.0053} & 71.85{\tiny$\pm$1.40} \\
    \bottomrule
  \end{tabular}
  \caption{\textbf{At half the parameter budget the encoder advantage grows
  rather than shrinks.} The \texttt{nano} iso-parameter band, same protocol as
  Table~\ref{tab:modern}. $^{\ddagger}$Top-1 uses the neutral
  Spiking-CNN judge. Never compare figures across bands.}
  \label{supp:tab:nano}
\end{table}

\section{Second Corpus}
\label{supp:second-corpus-full}

\looseness=-1 \cref{sec:second-corpus} summarizes the finding; Table~\ref{supp:tab:corpus2}
gives the full comparison, normalized by the achievable gain since the amount
of restoration available differs by an order of magnitude between the two
corruption models (LIMODENet gains $13.59$\,dB over its input under Gaussian
noise but only $1.66$ under pixelate). On that normalized scale the two
structured corruptions (EuroSAT/DVB-S2X and pixelate) agree with each other on
both baselines, across different datasets and corruption processes, while the additive-noise corruption is the outlier: two orders of magnitude against the skip architecture, and a factor of five against the same-class autoencoder.

\begin{table}[t]
  \centering
  \footnotesize
  \setlength{\tabcolsep}{4pt}
  \begin{tabular}{lcccc}
    \toprule
    Reconstructor & Params & PSNR (dB) & $\Delta$ & \% of gain \\
    \midrule
    \multicolumn{5}{l}{\emph{additive noise}: ImageNet-C \texttt{gaussian\_noise},
      input $15.52$\,dB} \\
    ~~\textbf{LIMODENet-AE} & 0.77M & \textbf{29.112$\pm$0.003} & --- & --- \\
    ~~U-Net (skips)         & 0.75M & 29.107$\pm$0.005 & $+0.005$ & $0.03\%$ \\
    ~~CNN-AE                & 0.75M & 28.901$\pm$0.005 & $+0.211$ & $1.55\%$ \\
    \midrule
    \multicolumn{5}{l}{\emph{information-destroying}: \texttt{pixelate}
      $4\times$, input $24.63$\,dB} \\
    ~~\textbf{LIMODENet-AE} & 0.77M & \textbf{26.288$\pm$0.005} & --- & --- \\
    ~~U-Net (skips)         & 0.75M & 26.180$\pm$0.005 & $\mathbf{+0.108}$ & $\mathbf{6.51\%}$ \\
    ~~CNN-AE                & 0.75M & 26.144$\pm$0.007 & $+0.145$ & $8.72\%$ \\
    \midrule
    \multicolumn{5}{l}{\emph{reference}: EuroSAT/DVB-S2X $1$\,dB$/100$q} \\
    ~~U-Net (skips)         & 0.75M & --- & $+1.07$ & $4.90\%$ \\
    ~~CNN-AE                & 0.75M & --- & $+1.75$ & $7.97\%$ \\
    \bottomrule
  \end{tabular}
  \caption{\looseness=-1 \textbf{On a second corpus, the advantage over a skip architecture
  tracks whether the corruption destroys information or merely adds noise.}
  PatternNet at $128$\,px, $3$ seeds, $60$ epochs, val-selected. ``\% of gain''
  normalises the margin by LIMODENet's own PSNR improvement over its input,
  since absolute margins are not comparable when the amount of recoverable
  signal differs by an order of magnitude. Seed deviations are $0.003$--$0.005$\,dB under \emph{both} corruptions, which makes non-overlap trivially easy to achieve, so read the table on effect size rather than on separation.}
  \label{supp:tab:corpus2}
\end{table}

\section{Complexity in Context}
\label{supp:complexity}

\looseness=-1 \cref{sec:limitations} summarizes the honest reading (parameter-efficient,
not compute-efficient); Table~\ref{supp:tab:complexity} gives the full comparison against compact onboard baselines and flight-proven references, with input resolutions noted.

\begin{table}[h]
  \centering
  \footnotesize
  \setlength{\tabcolsep}{3pt}
  \begin{tabular}{lccccc}
    \toprule
    Model & In. & Params & MACs & Lat. & Top-1 \\
    \midrule
    \multicolumn{6}{l}{\emph{LIMODENet family (ours)}}\\
    \quad nano   & $64^2$ & 0.39M & 149M  & 2.23 & 95.65 \\
    \quad tiny   & $64^2$ & 0.69M & 262M  & 2.41 & 96.54 \\
    \quad small  & $64^2$ & 1.55M & 581M  & 2.23 & 96.93 \\
    \quad base   & $64^2$ & 2.73M & 1027M & 2.37 & 97.26 \\
    \quad big    & $64^2$ & 5.80M & 2466M & 3.58 & 97.54 \\
    \quad \textbf{tiny-AE} & $128^2$ & 0.77M & 1.21G & 2.54 & --- \\
    \midrule
    \multicolumn{6}{l}{\emph{Compact onboard baselines}~\cite{le2026onboardvit}}\\
    \quad ResNet-14        & $64^2$  & 0.20M & 118M & --- & 93.88 \\
    \quad EfficientViT-M2  & $224^2$ & 3.96M & 204M & --- & 98.76 \\
    \quad MobileViTV2      & $256^2$ & 4.39M & 1.84G & --- & 99.09 \\
    \midrule
    \multicolumn{6}{l}{\emph{Flight-proven / proposed onboard}}\\
    \quad $\Phi$-Sat-2~\cite{guerrisi2023phisat2} & --- & --- & 81M & --- & --- \\
    \quad CloudScout~\cite{giuffrida2022phisat1} & $192^2$ & 0.16M & 4.67G & --- & --- \\
    \quad RepSViT~\cite{pang2024repsvit} & --- & 3.77M & 600M & --- & --- \\
    \bottomrule
  \end{tabular}
  \caption{\looseness=-1 \textbf{Complexity in context.} Latency is single-image (batch~1) on one RTX~4070 Laptop, FP32, median of 200 timed runs; the tiny-AE value is the measurement of Table~\ref{tab:modern}. Below about $1$\,GMac this GPU is launch-bound at batch~1, which is why the width ladder is nearly flat ($2.2$--$2.4$\,ms) and only \emph{big} separates. Our restoration model
  ($1.21$\,G MACs) sits \emph{between} two networks that have flown, i.e.\
  ${\sim}3.9\times$ below flight-proven demand. Note the honest reading of the
  top block: at $64^2$ LIMODENet-tiny already spends more MACs than EfficientViT-M2~\cite{le2026onboardvit} does at $224^2$, so we claim parameter efficiency and a
  viable absolute budget, \emph{not} compute efficiency. Input resolutions,
  training protocols and accuracy sources differ across blocks; rows are for
  order-of-magnitude context, not a controlled comparison.}
  \label{supp:tab:complexity}
\end{table}

\begin{figure*}[t]
  \centering
  \includegraphics[width=\textwidth]{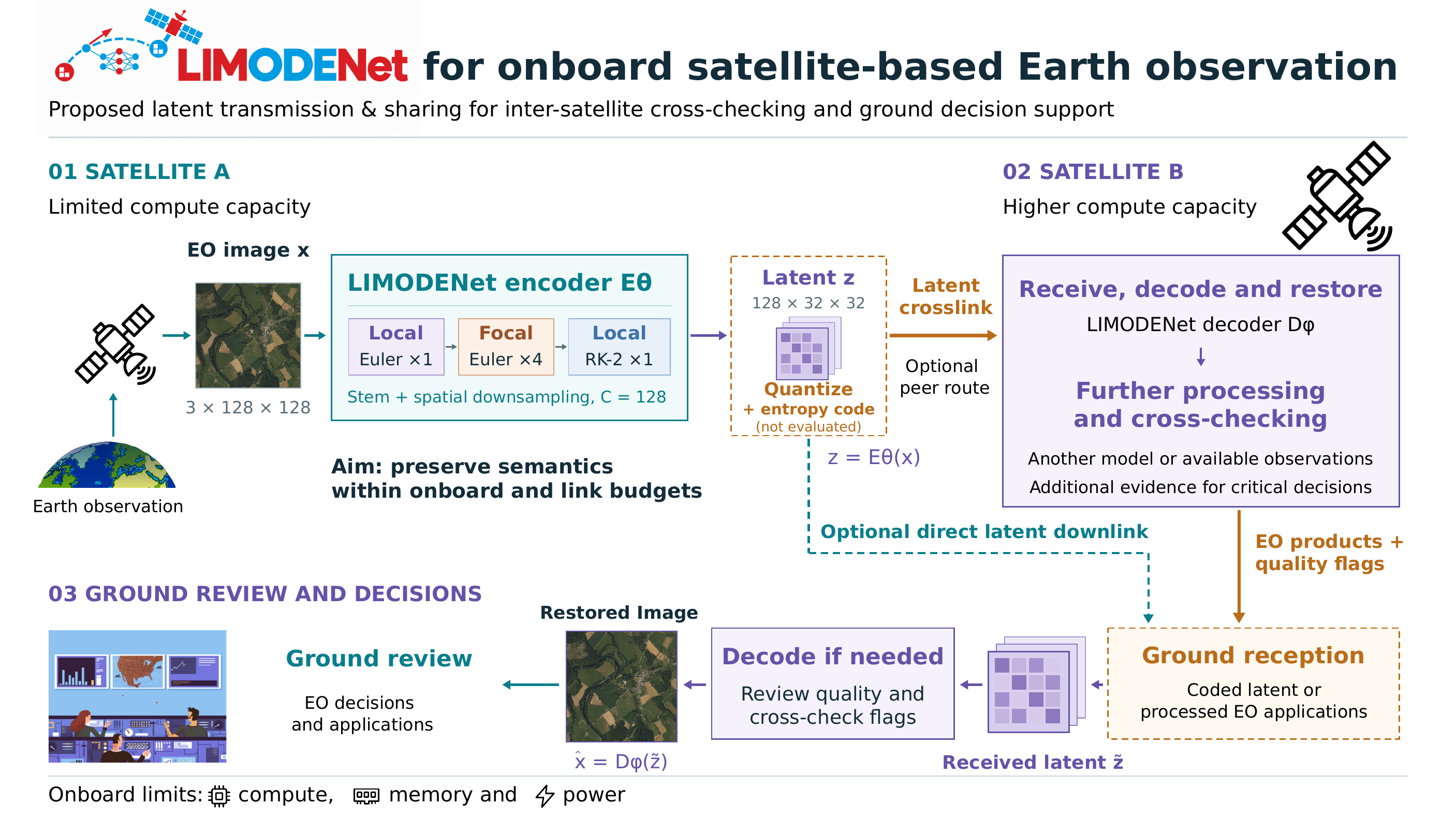}
  \caption{\textbf{The deployment setting our restoration experiments stand in
  for.} Satellite~A encodes an observation with the LIMODENet encoder
  $E_\theta$, quantises and entropy-codes the latent, and either crosslinks it
  to a better-provisioned satellite~B or downlinks it; B decodes with
  $D_\varphi$, cross-checks the restored scene, and forwards EO products with
  quality flags. The figure is a proposal. \emph{Measured in this paper:} only the
  $x\!\to\!E_\theta\!\to\!D_\varphi\!\to\!\hat{x}$ path, with degradation applied
  to pixels before the encoder. \emph{Not measured:} the latent quantiser and
  entropy coder, the crosslink, the effect of channel noise applied to $z$
  rather than to $x$, and any cross-checking logic.}
  \label{supp:fig:onboard}
\end{figure*}

\section{Future Work}
\label{supp:onboard}

\looseness=-1 Every experiment in this paper degrades \emph{pixels} on the ground and measures restoration offline. \cref{supp:fig:onboard} draws the operational setting those experiments stand in for, so the boundary between what we measured and what we propose is explicit. It is a design sketch, not a result: no part of the multi-satellite loop has been run, and nothing in the main paper depends on it.

\paragraph{Why this is the right target, and why it is future work.}
Three properties of the model are only worth their cost in this setting, which is why we name it rather than leave the motivation implicit. First, the constraint that costs us accuracy against unconstrained restorers (\cref{sec:conclusion}) buys spiking conversion with zero blocked operations, a trade that is rational only where power rather than accuracy binds, which is exactly the onboard segment. Second, the asymmetric split in the figure, a cheap encoder on the sensing platform and a decoder wherever compute is available, is the deployment shape that a $0.69$\,M-parameter
encoder with a $1.21$\,G-MAC restoration path
(Table~\ref{supp:tab:complexity}) actually fits. Third, the sub-threshold regime our results live in ($1$--$2$\,dB DVB-S2X, \cref{supp:abovecliff}) is a link-budget regime rather than an image-corruption one: above the decoding threshold every model ties, so the encoder earns its place only where the \emph{channel} is the bottleneck.

\paragraph{What must be measured before any of this can be claimed.}
\looseness=-1 We list the gaps in the order we consider them blocking. \emph{(i) The latent is not currently a compression win.} At $128\!\times\!32\!\times\!32$ against a $3\!\times\!128\!\times\!128$ input, $z$ holds $2.67\times$ as many values as the image it came from; any bit saving must come from the quantiser and entropy coder in the figure, which we have neither implemented nor rate--distortion evaluated. Until then ``semantic compression'' is a hypothesis, and a learned-compression baseline (e.g.\ a hyperprior codec at matched bitrate) would settle it.
\emph{(ii) Degradation in the right place.} Our corpus perturbs pixels; the
figure perturbs the latent in transit. These are different operators, and the
ordering of encoders under one does not transfer to the other by argument.
\emph{(iii) Robustness of $D_\varphi$ to a corrupted $\tilde{z}$.} The contractivity result (\cref{supp:p3}) bounds how perturbations propagate \emph{within} the encoder; it predicts, but does not demonstrate, graceful decoder degradation under latent bit errors.
\emph{(iv) End-to-end onboard measurement.} Our energy figures are a $45$\,nm SynOps proxy (\cref{supp:neuromorphic}); onboard viability needs wall-clock and joules on the target part, and a spiking decoder we do not have, so the neuromorphic result here is scoped to the classifier.
\emph{(v) The cross-checking step is entirely unspecified.} Satellite~B's
``additional evidence'' box names a system function, not an algorithm we have
designed or evaluated.

Together these define the future work rather than a caveat on this one: the present contribution is the constrained-optimal restoration encoder, and \cref{supp:fig:onboard} states where we intend to test whether that constraint pays operationally.

\end{document}